\documentclass{article}

\usepackage{microtype}
\usepackage{graphicx}
\usepackage{subfigure}
\usepackage{booktabs} % for professional tables
\usepackage{setspace}

\usepackage{hyperref}

\usepackage[accepted]{icml2019}
\usepackage[utf8]{inputenc} % allow utf-8 input
\usepackage[T1]{fontenc}    % use 8-bit T1 fonts

\usepackage{amssymb}
\usepackage{algorithm}
\usepackage{algpseudocode}

\usepackage{url}            % simple URL typesetting
\usepackage{booktabs}       % professional-quality tables
\usepackage{nicefrac}       % compact symbols for 1/2, etc.
\usepackage{microtype}      % microtypography
\usepackage{amsmath}
\usepackage{amsthm}
\usepackage{bm,bbm}

\usepackage{xcolor}

\usepackage{placeins}

\hypersetup{%
  colorlinks=true,% hyperlinks will be black
  urlcolor=blue,%
  linkcolor=blue,%
  citecolor=blue,%
  linkbordercolor=blue,% hyperlink borders will be red
  pdfborderstyle={/S/U/W 1}% border style will be underline of width 1pt
}

\usepackage{enumitem}
\setlist{leftmargin=5.5mm}
\setitemize{leftmargin=*,itemsep=0.25em,label=\raisebox{0.25ex}{\tiny$\bullet$}}

\usepackage[labelfont=footnotesize, textfont=footnotesize]{caption}

\renewcommand{\thetable}{\arabic{table}}

\newtheorem{lem}{Lemma}

\newtheorem{prop}{Proposition}

\theoremstyle{definition}

\newtheorem{defn}{Definition}

\newtheorem{example}{Example}
\DeclareMathOperator*{\argmin}{\arg\!\min}

\newcommand{\mean}{\mathbb{E}}
\newcommand{\ExpVal}[2]{\mean{}\left[ #2 \right]}
\newcommand{\EE}[1]{\ExpVal{}{#1}}

\newcommand{\KL}{\textnormal{D}_{\scalebox{.6}{\textnormal KL}}}

\newcommand{\calX}{\mathcal{X}}

\newcommand{\bx}{\bm{x}}

\renewcommand{\tilde}{\widetilde}

\newcommand{\Reals}{\mathbb{R}}

\newcommand{\defined}{\triangleq}

\newcommand{\indicator}[1]{\mathbb{I}{#1}}

\newcommand{\PXa}{P_0}
\newcommand{\PXb}{P_1}

\newcommand{\PYhata}{P_{\hat{Y}|S=0}}
\newcommand{\PYhatb}{P_{\hat{Y}|S=1}}

\newcommand{\PXp}{\tilde{P}_0}

\newcommand{\M}{\mathsf{M}(\PXa)}

\newcommand{\W}{P_{\hat{Y}|X}}

\newcommand{\RNum}[1]{\uppercase\expandafter{\romannumeral #1\relax}}

\newcommand{\gitrepo}[0]{\hyperlink{http://github.com/ustunb/ctfdist}{http://github.com/ustunb/ctfdist}}

\newcommand{\cell}[2]{\setlength{\tabcolsep}{0pt}\begin{tabular}{#1}#2 \end{tabular}}

\newcommand{\sccell}[2]{\setlength{\tabcolsep}{0pt}\textsc{\begin{tabular}{#1}#2 \end{tabular}}}

\usepackage{eqnarray,mathtools}

\icmltitlerunning{Avoiding Disparate Impact with Counterfactual Distributions}

\begin{document}

\twocolumn[
\icmltitle{Repairing without Retraining:\\Avoiding Disparate Impact with Counterfactual Distributions}

\icmlsetsymbol{equal}{*}

\begin{icmlauthorlist}
\icmlauthor{Hao Wang}{seas}
\icmlauthor{Berk Ustun}{seas}
\icmlauthor{Flavio P. Calmon}{seas}
\end{icmlauthorlist}

\icmlaffiliation{seas}{Harvard University, MA, USA}

\icmlcorrespondingauthor{Hao Wang}{hao\_wang@g.harvard.edu}
\icmlcorrespondingauthor{Berk Ustun}{berk@seas.harvard.edu}
\icmlcorrespondingauthor{Flavio P. Calmon}{flavio@seas.harvard.edu}

\icmlkeywords{fairness, classification, disparate impact, preprocessing, optimal transport, information theory}

\vskip 0.1in
]

\printAffiliationsAndNotice{}

\begin{abstract}
When the performance of a machine learning model varies over groups defined by sensitive attributes (e.g., gender or ethnicity), the performance  disparity can be expressed in terms of the probability distributions of the input and output variables over each group.
In this paper, we exploit this fact to reduce the disparate impact of a fixed classification model over a population of interest.
Given a black-box classifier, we aim to eliminate the performance gap by perturbing the distribution of input variables for the disadvantaged group.
We refer to the perturbed distribution as a \emph{counterfactual distribution}, and characterize its properties for common fairness criteria.
We introduce a descent algorithm to learn a counterfactual distribution from data. 
We then discuss how the estimated distribution can be used to build a data preprocessor that can reduce disparate impact without training a new model. 
We validate our approach through experiments on real-world datasets, showing that it can repair different forms of disparity without a significant drop in accuracy.
\end{abstract}

\section{Introduction}
\label{Sec::Introduction}

A machine learning model has \emph{disparate impact} when its performance changes across groups defined by \emph{sensitive attributes} such as race or gender \citep{barocas2016disparate}. Recent work has shown that models can exhibit significant performance disparities between groups \citep[see e.g.][]{angwin2016machine,dastin2018amazon}. Such disparities have led to a plethora of research on fair machine learning, focusing on how disparate impact arises \citep[][]{chen2018my,datta2016algorithmic}, how it can be measured \citep{vzliobaite2017measuring,pierson2017fast,simoiu2017problem,kilbertus2017avoiding, kusner2017counterfactual,galhotra2017fairness}, and how it can be mitigated \citep[][]{feldman2015certifying,corbett2017algorithmic,zafar2017parity, calmon2017optimized,menon2018cost,canetti2019soft}. Despite these developments, disparate impact remains difficult to avoid in a large class of real-world applications where:
\begin{itemize}[align=left]

\item  Models are procured from a third-party vendor who has the data or technical expertise required for model development \citep{guszcza2018why}.

\item  Models are deployed on a population where the data distribution differs from the distribution of the training data \citep[i.e., due to dataset shift,][]{sugiyama2017dataset}. 

\end{itemize}
In such settings, disparate impact is challenging to address, let alone understand. Users typically have black-box access to the model (e.g., via a prediction API), may not have access to the training data (e.g., due to privacy concerns or intellectual property rights), and may not be able to draw conclusions from the training data (e.g., due to distributional shifts in deployment). 

In this paper, we aim to mitigate disparate impact in such settings. Our object of interest is a hypothetical distribution of input variables that minimizes disparate impact in the model's \emph{deployment population}. We refer to this distribution as a \emph{counterfactual distribution}. As we will show, an information-theoretic analysis of counterfactual distributions has much to offer. Given a fixed classifier, disparate impact can be expressed in terms the distributions of input and output variables between groups (cf. Figure~\ref{Fig::ProbailitySimplexNoDiscrimination}). In turn, a counterfactual distribution can be obtained by repeatedly perturbing the distribution of input variables until a specific measure of disparity is minimized over the deployment population. The counterfactual distribution can then be used to repair the model so that it no longer exhibits disparate impact in deployment.

\begin{figure}[t]
\centering
\includegraphics[width=\linewidth]{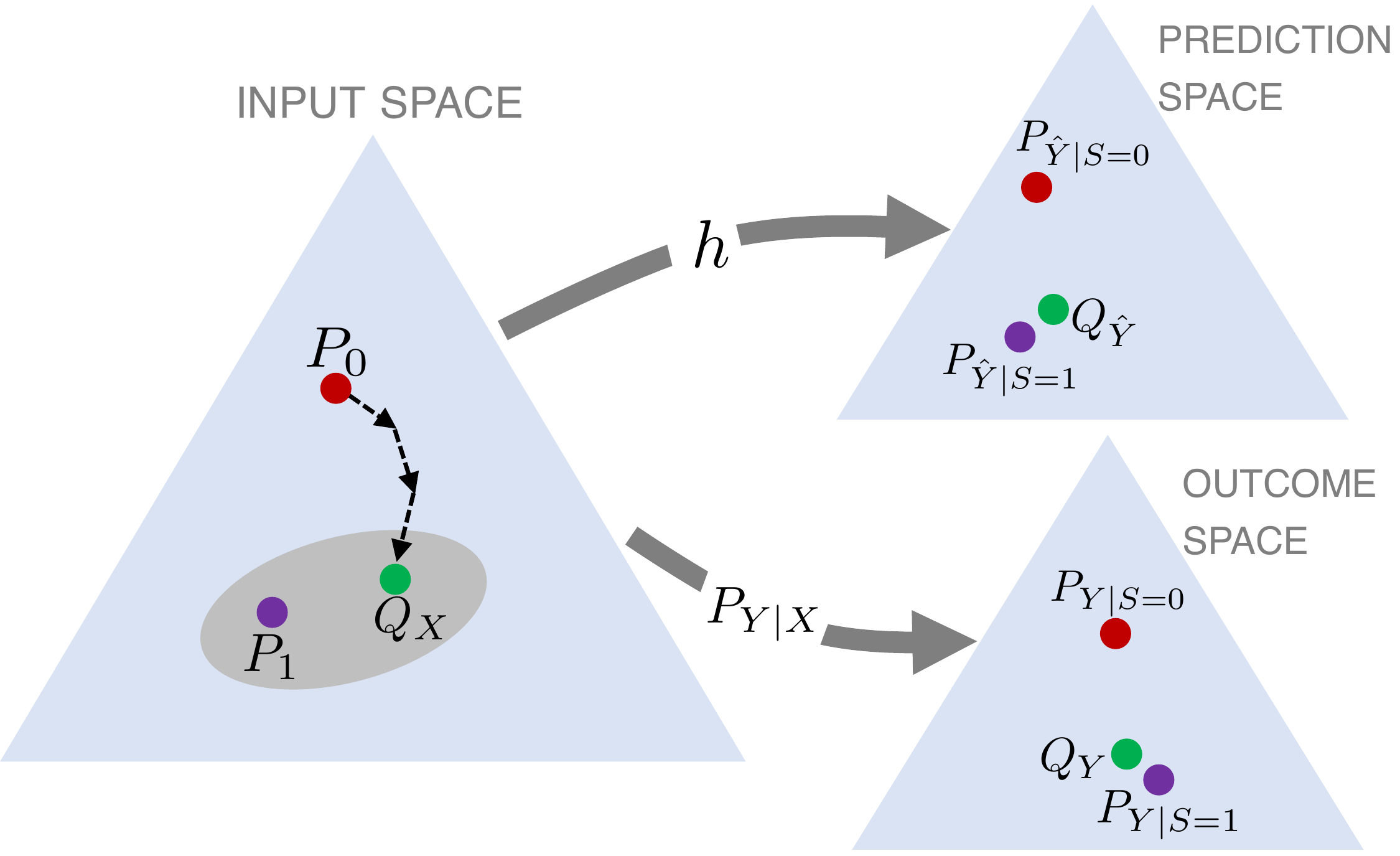}
\caption{Illustration of probability distributions affecting the disparate impact of a fixed classification model $h$. Here, $\PXa$ and $\PXb$ denote the distributions of input variables for groups where $S=0$ and $S = 1$, respectively. Disparate impact is a function of the distributions of predicted outcomes ($P_{\hat{Y}|S=0}$, $P_{\hat{Y}|S=1}$) and true outcomes ($P_{Y|S=0}$, $P_{Y|S=1}$). A \emph{counterfactual distribution} $Q_X$ is a perturbation of $P_0$ that minimizes a specific measure of disparity. If disparate impact persists under $Q_X$, there may be irreconcilable differences between the groups (i.e., $P_{Y|X,S=0}\neq P_{Y|X,S=1}$, see Prop.~\ref{prop::ctf_posi_dist_diff}). The counterfactual distribution may not be unique, as illustrated by the shaded ellipse.}
\label{Fig::ProbailitySimplexNoDiscrimination}
\end{figure}

\paragraph{Contributions}

The main contributions of this paper are:
\begin{enumerate}

\item We introduce a theoretical framework to mitigate performance disparities for a black-box classifiers in deployment.

\item We develop machinery to learn counterfactual distributions from data. Our tools recover a counterfactual distribution using a descent procedure in the simplex of probability distributions. We prove that influence functions can be used to compute a gradient in this setting, and derive closed-form estimators that enable efficient computation of influence functions for several (group) fairness criteria. 

\item We design pre-processing methods that use counterfactual distributions to repair a black-box classifier in a deployment population without the need to train a new model. The proposed method only affects individuals in a specific group in a way that improves their outcomes (on average).

\item We validate our procedure by repairing classifiers trained with real-world datasets. Our results demonstrate how counterfactual distributions can help mitigate disparate impact in real-world applications.
\end{enumerate}

\paragraph{Use Cases}

Our approach provides a way to repair classifiers in domains where treatment disparity is legal and ethical\footnote{Treatment disparity is not illegal nor unethical in all applications of machine learning where fairness is important. In medicine, for example, it is acceptable to use sensitive attributes in prediction. In lending, the Equal Credit Opportunity Act permits the use of age in credit scoring models.}. 
Illustrative use cases include:
\begin{itemize}

    \item A hospital purchases a readmissions prediction model that satisfies equal opportunity for different genders in the training data but violates it in deployment;
    
    \item  A rural clinic purchases a classification model to detect bone fractures in x-rays and discovers that patients with a certain physical trait have high FPR; 
    
    \item A bank enters a new market and discovers its credit score underperforms on customers over 60 years of age.

\end{itemize}
The tools in this paper will be able to scrutinize and repair performance disparities in all three settings --- regardless of whether the model directly uses the sensitive attribute. More importantly, they will be able to achieve parity-based notions of fairness in a way that: (i) only affects one group; (ii) benefits the group it affects (on average); and (iii) incentivizes individuals to reveal their sensitive attributes at prediction time. The latter two points (i.e., do-no-harm and opt-in) are important elements of ethical treatment disparity \citep[see e.g.,][]{lipton2018does, ustun2019fairness}.

\paragraph{Related Work}

We develop a theoretical framework that is used to design methods to determine counterfactual distributions in practice. We then use counterfactual distributions to design optimal transport-based pre-processing methods for ensuring fairness. In this regard, the closest work to ours are those of \citet{feldman2015certifying,johndrow2017algorithm,del2018obtaining}, which propose methods to control specific disparate impact metrics via optimal transport. These methods differ from ours in that they (i) focus on reducing measures of disparity related to predicted outcomes; (ii) map the input variable distributions across \emph{all} sensitive groups to a common distribution. More broadly, our approach differs from other model-agnostic approaches to mitigate disparate impact \citep[e.g., pre-processing methods such as][]{kamiran2012data,calmon2017optimized} in that it does not require access to the training data, and does not require training a new model. 

The term ``counterfactual distribution'' is often used to describe different kinds of hypothetical effects. In the statistics and economics literature \citep[see e.g.,][]{balke1994counterfactual,dinardo1995labor,rubin2005causal,fortin2011decomposition,chernozhukov2013inference,johansson2016learning,peters2017elements,fisher2018visually}, a counterfactual distribution refers to a hypothetical distribution of an \textit{outcome variable} given a specific distribution of input variables (e.g., the distribution of wages (outcome variable) for young workers if young workers had the same qualifications as older workers). The counterfactual distribution in this work describes a different kind of effect --- i.e., a distribution of input variables to minimize disparate impact --- and, consequently, must be derived using a different set of tools.

\paragraph{Additional Resources}

We provide a software implementation of our tools at \gitrepo{}. This paper extends work that was first presented at the NeurIPS WESGAI Workshop \citep{wang2018nips}.

\section{Framework}
\label{Sec::Framework}

In this section, we formally define counterfactual distributions and discuss their properties.

\paragraph{Preliminaries}

We consider a standard classification task where the goal is to predict a binary outcome variable $Y \in \{0,1\}$ using a vector of input variables $X=(X_1,\dots,X_d)\in\mathcal{X}$ drawn from the probability distribution $P_X.$ We are given a black-box classifier $h: \mathcal{X} \to [0,1]$. We assume that $h(\bx) \in \{0,1\}$ if the classifier outputs a predicted outcome (e.g., SVM) and $h(\bx) \in [0,1]$ if it outputs a predicted probability (e.g., logistic regression).

We evaluate differences in the performance of the classifier with respect to a \emph{sensitive attribute} $S \in \{0,1\}$ with distribution $P_S$. We refer to the subset of individuals with $S = 0$ and $S = 1$ as the \emph{target} and \emph{baseline} groups, respectively. We denote their distributions of input variables as $\PXa\defined P_{X|S=0},$ and $\PXb \defined P_{X|S=1}$. Likewise, we let $P_{\hat{Y}|S=0}(1) \defined \EE{h(X)|S=0}$ and $ P_{\hat{Y}|S=1}(1) \defined\EE{h(X)|S=1}.$  For the sake of clarity, we assume that $S$ is not an input variable to the model. Note, however, that our tools and results can be extended to settings where $S$ an input variable for the classifier.

\paragraph{Disparity Metrics}

We measure the performance disparity between groups in terms of a \emph{disparity metric}. 
Formally, a disparity metric is a  mapping $\mathsf{M}: \mathcal{P} \to \Reals$ where $\mathcal{P}$ is the set of probability distributions over $\mathcal{X}$.
We provide examples of $\M$ for common fairness criteria in Table~\ref{tabel:ExpDiscMet}. Note that we write disparity metrics as $\M$ since they can be expressed as a function of $\PXa$ once the classifier and the distributions $P_{Y|X,S}$, $\PXb$, and $P_S$ are fixed.%
\footnote{As we show in Appendix~\ref{append::fact_metrics}, the disparity metric can be expressed as a function of $\PXa$ once the following objects are fixed: $h$, the classifier; $P_{Y|X,S}$ the distribution of the true outcomes given input variables and sensitive attribute; $\PXb$, the distribution of input variables over the baseline group; and $P_S$ the distribution of the sensitive attribute.}

\begin{table*}[t]
\small
\centering
%\resizebox{0.9\textwidth}{!}{
\renewcommand{\arraystretch}{1.25}
\begin{tabular}{lll}

\toprule
\textsc{Performance Metric} & \textsc{Acronym} & \textsc{Disparity Metric} \\
\toprule  

Statistical Parity &
SP &
$\Pr(\hat{Y}=0|S=0)-\Pr(\hat{Y}=0|S=1)$ 

\\ \midrule 

False Discovery Rate & FDR & $\Pr(Y=0|\hat{Y}=1,S=0)-\Pr(Y=0|\hat{Y}=1,S=1)$

\\ \midrule 

False Negative Rate & FNR & $\Pr(\hat{Y}=0|Y=1,S=0)-\Pr(\hat{Y}=0|Y=1,S=1)$ 
\\ \midrule

False Positive Rate & FPR & $\Pr(\hat{Y}=1|Y=0,S=0)-\Pr(\hat{Y}=1|Y=0,S=1)$\\

\midrule 

Distribution Alignment & $\textrm{DA}_{\lambda}$ &  $\KL(\PYhata\|\PYhatb)+\lambda\KL(\PXa\|\PXb)$ \\ 

\bottomrule 
\end{tabular}
%}
\caption{Disparity metrics $\M$ for common fairness criteria \citep[see e.g.,][for a list]{romei2014multidisciplinary}. We assume that $S = 0$ attains the less favorable value of performance so that $\M\geq 0$. Our tools can be used for any fairness criterion that can be expressed as a convex combination of these metrics (e.g., equalized odds; see Proposition \ref{thm::Linearility_Inf_Cor}). \emph{Distribution Alignment} is a metric proposed in \citet{wang2018influence} that measures the disparity of predicted outcomes via the KL-divergence.}
\label{tabel:ExpDiscMet}
\end{table*}

\paragraph{Counterfactual Distributions}

A \emph{counterfactual distribution} is a hypothetical probability distribution of input variables for the target group that minimizes a specific disparity metric.
\begin{defn}
\label{defn::CounterfactualDist}
A counterfactual distribution $Q_X$ is a distribution of input variables for the target group such that:
\begin{align}
Q_X \in \argmin_{Q'_X\in\mathcal{P}} \left|\mathsf{M}(Q'_X)\right|, \label{Eq::CounterfactualSet}
\end{align}
where $\mathsf{M}(\cdot)$ is a given disparity metric and $\mathcal{P}$ is the set of probability distributions over $\mathcal{X}$.
\end{defn}

There exist several ways to resolve the performance disparity of a fixed classifier by perturbing the distributions of input variables of sensitive groups. For example, one could simultaneously perturb the input distributions for all groups to a ``midpoint'' distribution \citep[see e.g., the distributions considered by][to achieve statistical parity]{feldman2015certifying,johndrow2017algorithm,del2018obtaining}. 

While our tools could recover such distributions, we will purposely consider a counterfactual distribution that alters the input variables for a sensitive group that attains the less favorable performance (i.e., the target group $S = 0$). This choice reflects our desire to resolve the performance disparity by having the target group perform better, rather than having the baseline group perform worse. As we discuss later, this choice reduces the data requirements to estimate the counterfactual distribution and the individuals who are affected by the repair (i.e., this approach only produces a preprocessor that affects individuals where $S = 0$).

At this point, an observant reader may wonder why a counterfactual distribution for the target group is not simply the distribution of input variables over the baseline group (i.e., $Q_X  \equiv \PXb$). In fact, the distribution of input variables for the baseline group $\PXb$ \emph{is not necessarily} a counterfactual distribution when $P_{Y|X,S=0}\neq P_{Y|X,S=1}$. We illustrate this point with the following example.

\begin{example}
\label{Example::MajorityIsNotCTF}
%Recall that $\PXa\defined P_{X|S=0}$ and $\PXb\defined P_{X|S=1}$ are the corresponding input variables distributions across target group ($S=0$) and baseline group ($S=1)$. 
Consider a classification task where the input variables $X=(X_1,X_2)\in\{0,1\}^2$ are drawn from distributions such that $P_{X|S=s} = P_{X_1|S=s} \cdot P_{X_2| S=s}$ for $s \in \{0,1\}$ where:
\begin{align*}
\Pr(X_1 = 1 | S = 0) = 0.9, \; & \Pr(X_2 = 1 | S = 0) = 0.2,\\
\Pr(X_1 = 1 | S = 1) = 0.1, \; & \Pr(X_2 = 1 | S = 1) = 0.5.
\end{align*}  
Assume that the true outcome variables $Y$ are drawn from the conditional distributions:
\begin{align}
    \label{eq:defns_ex1}
    \begin{split}
    P_{Y|X,S=0}(1|\bx) &= \mathsf{logistic}(2x_1-2x_2),\\ 
    P_{Y|X,S=1}(1|\bx) &= \mathsf{logistic}(2x_1+4x_2-3).
    \end{split}
\end{align}
In this case, the Bayes optimal classifier for $S = 1$ is $h(\bx) = \indicator[x_2=1]$. Using the difference in FPR as the disparity metric, $h$ achieves $\mathsf{M}(\PXa) = 25.1\%.$ In this case, setting $\PXa \gets \PXb$ would achieve a disparity of  $\mathsf{M}(\PXb)=43.6\%$. In contrast, we can achieve a disparity metric of $\mathsf{M}(Q_X)=0.0\%$ for a counterfactual distribution such that 
\begin{align*}
Q_X(0,0) &= 0.50, \qquad & Q_X(0,1) &= 0.09,\\ 
Q_X(1,0) &= 0.41, \qquad & Q_X(1,1) &= 0.00.
\end{align*}
%$Q_X(0,0) = 0.50, Q_X(0,1) = 0.09, Q_X(1,0) = 0.41, Q_X(1,1) = 0.00.$
\end{example}

Example \ref{Example::MajorityIsNotCTF} shows that counterfactual distributions may be non-trivial when the conditional distributions of $Y$ given $X$ differ across groups (i.e., $P_{Y|X,S=0}\neq P_{Y|X,S=1})$. In particular, the condition  $P_{Y|X,S=0}\neq P_{Y|X,S=1}$ will always hold whenever counterfactual distributions do not completely eliminate the disparity between groups. We formalize this statement in the next proposition.
\begin{prop}
\label{prop::ctf_posi_dist_diff}
If $\mathsf{M}(Q_X)>0$ where $Q_X$ is a counterfactual distribution for a disparity metric in Table~\ref{tabel:ExpDiscMet}, then $P_{Y|X,S=0} \neq P_{Y|X,S=1}$.
\end{prop}

Proposition~\ref{prop::ctf_posi_dist_diff} illustrates how a counterfactual distribution can be used to detect cases where a classifier exhibits an irreconcilable performance disparity between groups  -- i.e., a disparity that cannot be resolved by perturbing the distributions of input variables for the target group. The result complements various impossibility results on inevitable trade-offs between groups \citep[see e.g.,][]{chouldechova2017fair,kleinberg2016inherent,pleiss2017fairness}. It also provides a sufficient condition that can inform the need for treatment disparity \citep[see e.g., of][]{kleinberg2018algorithmic,dwork2018decoupled,lipton2018does,ustun2019fairness}.

\section{Methodology}
\label{Sec::Methodology}

In this section, we present information-theoretic tools to learn counterfactual distributions from data. We first demonstrate how influence functions provide a natural ``descent direction'' in this setting. Next, we derive closed-form expressions for the influence functions of the disparity metrics in Table \ref{tabel:ExpDiscMet}. Lastly, we present a descent procedure that combines these results to learn a counterfactual distribution for the deployment population. 

\subsection{Measuring the Descent Direction} 

In what follows, we describe how to reduce the value of a disparity metric by perturbing the distribution of input variables over the target group $\PXa$. 

We start by formally defining the local perturbation of an input distribution.
\begin{defn}
\label{Def:perturb}
The perturbed distribution $\PXp$ over the target group ($S=0$) is given by 
\begin{align}
\label{Eq::Perturb_dist}
\PXp(\bx)\defined \PXa(\bx)(1+\epsilon f(\bx)),\quad \forall \bx\in\mathcal{X}
\end{align}
where $f: \mathcal{X}\to\Reals$ is a perturbation function from the class of all functions with zero mean and unit variance w.r.t. $P_0$, and $\epsilon>0$ is a positive scaling constant chosen so that $\PXp$ is a valid probability distribution.
\end{defn}

Here, $f(\bx)$ represents a direction in the probability simplex while $\epsilon$ represents the magnitude of perturbation \citep[see e.g.,][for other  applications of local perturbations of measures in information theory]{borade2008euclidean,anantharam2013maximal,huang2014efficient}. 

As we will see shortly, the direction of steepest descent for disparate impact can be measured using an \emph{influence function} \citep[see e.g.,][for other uses in machine learning]{huber2011robust,koh2017understanding}.
\begin{defn}
\label{defn::inf_func}
For a disparity metric $\mathsf{M}(\cdot)$,  the influence function $\psi: \mathcal{X}\to \Reals$ is given by
\begin{align}
\label{eq::defnInfluence}
    \psi(\bx)
    \defined \lim_{\epsilon\to 0}\frac{\mathsf{M}\left((1-\epsilon)\PXa+\epsilon \delta_{\bx}\right) -\mathsf{M}(\PXa)}{\epsilon}
\end{align}
where $\delta_{\bx}(\bm{z}) = \indicator[\bm{z} = \bx]$ is the delta function at $\bx$.
\end{defn}
Intuitively, given a sufficiently large dataset from the deployment population, the influence function approximates the change in a disparity metric when a sample $\bx \in \mathcal{X}$ from the target group is removed (or added) to the dataset.

In Proposition \ref{thm::InfFunc_Corr_Func}, we show that perturbing the distribution $\PXa$ along the direction defined by $-\psi(\bx)$ produces the largest local decrease of the disparity metric. That is, $-\psi(\bx)$ reflects the direction of steepest descent in disparate impact. 

\begin{prop}
\label{thm::InfFunc_Corr_Func}
Given a disparity metric $\mathsf{M}(\cdot)$, we have that
\begingroup
\small
\begin{equation}
\label{equa::corrFuncInfFunc}
\begin{aligned}
    \argmin_{f(\bx)}\lim_{\epsilon \to 0} \frac{\mathsf{M}(\PXp) - \mathsf{M}(\PXa)}{\epsilon}
    =\frac{-\psi(\bx)}{\sqrt{\EE{\psi(X)^2|S=0}}},
\end{aligned}
\end{equation}
\endgroup
for any influence function $\psi:\mathcal{X}\to \Reals$ such that $\EE{\psi(X)^2|S=0}\neq 0$.
\end{prop}

%%%%%%%%%%%%
%Combination of Metrics
%%%%%%%%%%%%
Proposition \ref{thm::Linearility_Inf_Cor} shows that when disparity is measured using a linear combination of metrics, the influence function for the compound metric can be expressed as a linear combination of the influence functions for its components.
\begin{prop}
\label{thm::Linearility_Inf_Cor}
Given any convex combination of $K$ disparity metrics $\M = \sum_{i=1}^K\lambda_i\mathsf{M}_i(\PXa),$ the influence function of the compound disparity metric $\M$ has the form:
\begin{equation}
    \psi(\bx)=\sum_{i=1}^K\lambda_i\psi_i(\bx). 
\end{equation}
\end{prop}
Proposition \ref{thm::Linearility_Inf_Cor} allows us to consider a larger class of disparity measures than those shown in Table \ref{tabel:ExpDiscMet}. For instance, one can recover a counterfactual distribution to achieve equalized odds \citep{hardt2016equality} by using a convex combination of influence functions for FPR and FNR.

\subsection{Computing Influence Functions}

\newcommand{\shat}[1]{\hat{s}(#1)}
\newcommand{\yhat}[1]{\hat{y}_0(#1)}

We now present closed-form expressions for the influence functions of disparity metrics shown in Table~\ref{tabel:ExpDiscMet}. The expressions will be cast in terms of three classifiers:

\begin{itemize}
\item $h(\bx)$: the black-box classifier that we aim to repair;

\item $\shat{\bx}$: a classifier that uses the same input variables as $h$, but aims to predict the probability that an individual belongs to the baseline group, $P_{S|X}(1|\bx)$.

\item $\yhat{\bx}$: a classifier that uses the same input variables as $h$, but aims to predict the true outcome \emph{for individuals in the target group}, $P_{Y|X, S = 0}(1|\bx)$.
 
\end{itemize}

Given $h(\bx)$, we would train $\shat{\bx}$ and $\yhat{\bx}$ using an \emph{auditing dataset} $\mathcal{D}^{\text{audit}}=\{(\bx_i,y_i,s_i)\}_{i=1}^n$ drawn from the deployment population. With these three models in hand, we can then compute influence functions using  closed-form expressions shown in the following proposition.

\begin{prop}
\label{prop::inf_func_exp}
The influence functions for the disparity metrics in Table~\ref{tabel:ExpDiscMet} can be expressed as
\begingroup
\small
\begin{align*}
\psi^{SP}(\bx) &= -h(\bx) + \hat{\mu}_0,\nonumber\\
\psi^{FDR}(\bx) &= \frac{ h(\bx)(1-\yhat{\bx})-\nu_{0,1}h(\bx)}{\hat{\mu}_0}, \nonumber \\
\psi^{FNR}(\bx) &= \frac{(1-h(\bx))\yhat{\bx}-\gamma_{0,1}\yhat{\bx}}{\mu_0}, \nonumber\\
\psi^{FPR}(\bx) &= \frac{h(\bx)(1-\yhat{\bx})-\gamma_{1,0}(1-\yhat{\bx})}{(1-\mu_0)}, \nonumber\\
\psi^{DA}(\bx) &= \log\frac{\hat{\mu}_0(1-\hat{\mu}_1)}{(1-\hat{\mu}_0)\hat{\mu}_1} h(\bx)+\lambda\log\frac{1-\shat{\bx}}{\shat{\bx}} \nonumber \\
&- \hat{\mu}_0\log\frac{\hat{\mu}_0(1-\hat{\mu}_1)}{(1-\hat{\mu}_0)\hat{\mu}_1} - \lambda\EE{\log\frac{1-\shat{X}}{\shat{X}}\Bigg|S=0}, \nonumber\\
\intertext{where $\mu_s$, $\hat{\mu}_s$, $\gamma_{a,b}$, and $\nu_{a,b}$ are constants such that}
\mu_s &\defined \Pr(Y=1|S=s), \nonumber \\
\hat{\mu}_s &\defined \Pr(\hat{Y}=1|S=s), \nonumber\\
\gamma_{a,b}&\defined \Pr(\hat{Y}=a|Y=b,S=0),\nonumber\\
\nu_{a,b}&\defined \Pr(Y=a|\hat{Y}=b,S=0).\nonumber
\end{align*}
\endgroup
\end{prop}

\subsection{Learning Counterfactual Distributions from Data}
\label{sec::bey_loc_pert}

\newcommand{\resample}[2]{\textsc{Resample}(#1,#2)}
\newcommand{\discmetric}[2]{\mathsf{M}(#1\cup #2)}
\newcommand{\data}[1]{\mathcal{D}}
\newcommand{\weights}[1]{\bm{w}_{#1}}
\newcommand{\ifit}[1]{I_{#1}}
\newcommand{\datafit}[1]{\mathcal{D}_{#1}}
\newcommand{\ctfdatafit}[1]{\tilde{\mathcal{D}}_{#1}}
\newcommand{\istop}[1]{I_{#1}^{\textrm{stop}}}
\newcommand{\datastop}[1]{\mathcal{D}_{#1}}

So far we have shown that influence functions can be used to evaluate the direction of steepest descent of a disparity metric (Proposition~\ref{thm::InfFunc_Corr_Func}), and that the value of an influence function can be estimated using data from the deployment population (Proposition~\ref{prop::inf_func_exp}). 

Considering these results, one would expect that disparity could be minimized by repeatedly (i) perturbing the distribution in the direction of steepest descent \eqref{equa::corrFuncInfFunc}, and (ii) estimating the influence function at the new, perturbed distribution. Repeating these steps, we would recover an approximate solution to \eqref{Eq::CounterfactualSet} -- i.e., an approximate counterfactual distribution.

In Algorithm~\ref{alg:descent}, we formalize this intuition by presenting a descent procedure to recover a counterfactual distribution for a given disparity metric $\mathsf{M}(\cdot)$. The procedure is analogous to stochastic gradient descent in the space of distributions over $\calX$, where the resampling at each iteration corresponds to a gradient step.

\begin{algorithm}[httb]
\begingroup
\small
\caption{Distributional Descent.}
\label{alg:descent}
\begin{spacing}{1.1}
\begin{algorithmic}[*]
\State {\bfseries Input:}
\State \quad $h: \mathcal{X} \to [0,1]$ \hfill\Comment{classification model}

\State \quad $\data{} =\{(\bx_i,y_i,s_i)\}_{i=1}^n$ \hfill\Comment{data from deployment population}

\State \quad $\mathsf{M}(\cdot)$ \hfill\Comment{disparity metric}

\State \quad $\epsilon>0$ \Comment{step size} \vspace{0.25em}

\State \textbf{Initialize} \vspace{0.25em}

\State $\ifit{0} \gets \{i = 1,\ldots, n ~|~ s_i = 0\}$ 

\State $\datafit{0} \gets (\bx_i, y_i)$ for $i \in \ifit{0}$ \Comment{samples where $s_i = 0$}

\State $\datafit{1} \gets (\bx_i, y_i)$ for $i \not\in \ifit{0}$ \Comment{samples where $s_i \neq 0$}

\State $\weights{0} \gets [w_i]_{i \in \ifit{0}}$ where $w_i = 1.0$ \Comment{initialize weights}

\State $M \gets \discmetric{\datastop{0}}{\datastop{1}}$ 
\Comment{evaluate disparity metric}

\Repeat
\State $M^\textrm{old} \gets M$

\State $\psi_i \gets \psi(\bx_i)$ for $i \in \ifit{0}$ \Comment{compute $\psi(\bx_i)$ for points in $\datafit{0}$}

\State $w_i \gets (1 - \epsilon \psi_i)\cdot w_i$  for $i \in \ifit{0}$ 

\State $\ctfdatafit{0} \gets \resample{\datafit{0}}{\bm{w}_0}$

\State $M \gets \discmetric{\ctfdatafit{0}}{\datastop{1}}$
\Comment{evaluate disparity metric}

\Until{$M \geq M^\textrm{old}$}

\State \textbf{return:} $\weights{0}$, $\ctfdatafit{0}$ \Comment{samples from counterfactual distribution}%

\\ \hrulefill
\Procedure{Resample}{$\data{}$, $\bm{w}$}
\State \textbf{return:}  $|\mathcal{D}|$ points sampled from $\mathcal{D}$ using the weights $\bm{w}$
\EndProcedure
\end{algorithmic}
\end{spacing}
\endgroup

\end{algorithm}

Given a classifier $h$ and a dataset $\{(\bx_i, y_i, s_i)\}_{i=1}^n$ from the distribution $P_{X,Y,S}$, the procedure outputs a dataset drawn from a counterfactual distribution. 

The procedure pairs each point with a \emph{sampling weight} $w_i$, which is initialized as $w_i = 1.0$. At each iteration, it first computes the value of the influence function $\psi(\bx).$ Next, it updates the values of each sampling weight for each point in the target group as $(1-\epsilon\psi(\bx_i))\cdot w_i$, where $\epsilon$ is a user-specified step size parameter. The updated sampling weights represent the direction in which the distribution for the target group should be perturbed to reduce $\mathsf{M}(\cdot).$ The data points from the target group are then resampled with their sampling weights. The set of resampled points mimics one drawn from the perturbed distribution. 

The procedure determines if the classifier still has disparate impact at the end of each iteration by computing the value of $\mathsf{M}(\cdot)$ on the set of resampled points. These steps are repeated until $\mathsf{M}(\cdot)$ ceases to decrease. Once the procedure stops, it outputs: (i) dataset drawn from a counterfactual distribution; (ii) a set of sampling weights for each point from the target group, which can be used to draw samples from the counterfactual distribution.

In Figure~\ref{Fig::toy_descent_fpr}, we show the progress (and convergence) of Algorithm~\ref{alg:descent} when recovering a counterfactual distribution for a synthetic dataset described in Appendix~\ref{Appendix::Experiments}. 

\begin{figure}[!tbp]
  \centering
  \includegraphics[width=\linewidth]{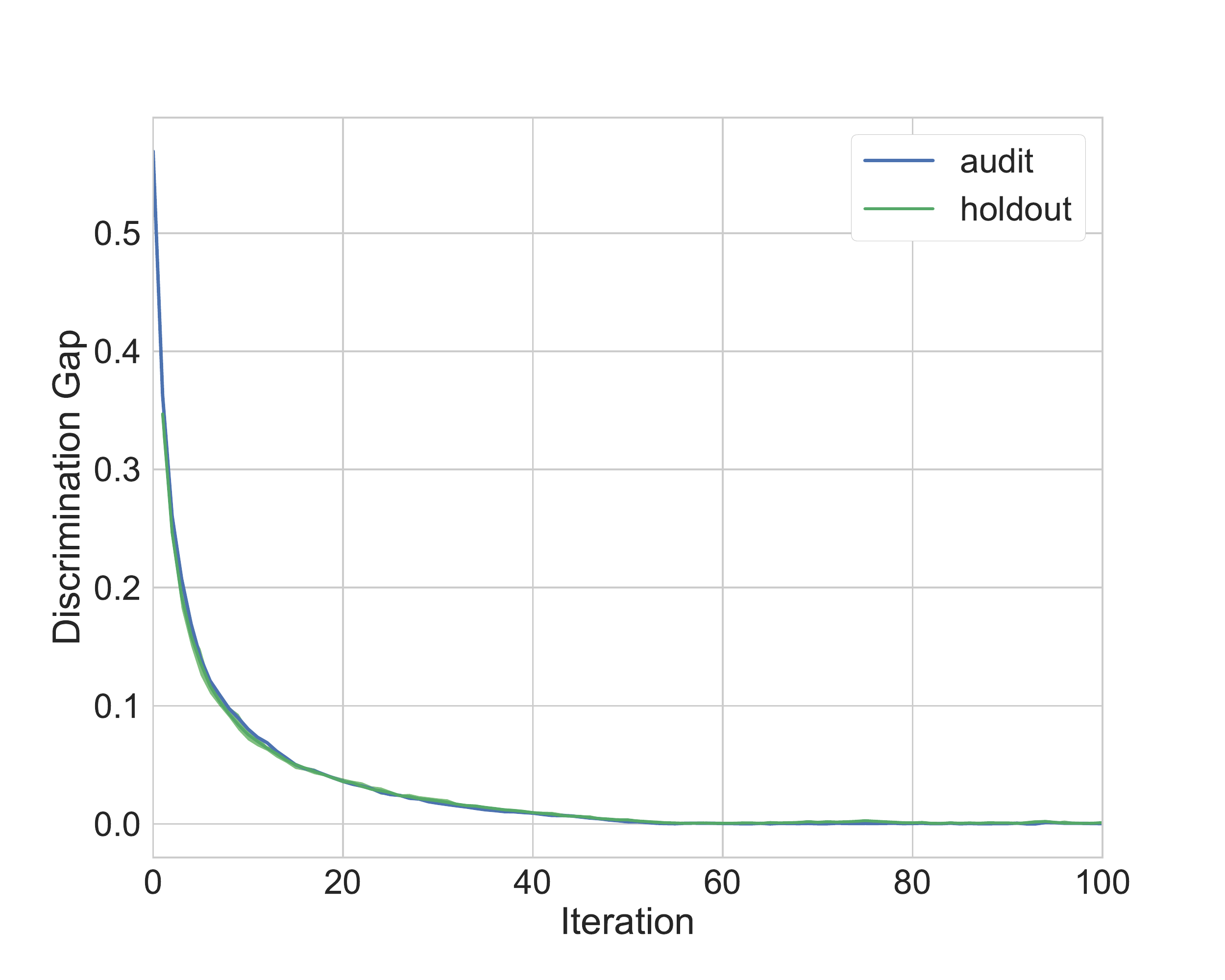}
  \caption{\small{Values of FPR for auditing dataset (blue) and holdout dataset (green), respectively, with each iteration in distributional descent for a synthetic dataset. Here, the procedure converges to a counterfactual distribution in 50 iterations. We show additional steps for the sake of illustration.}}
  \label{Fig::toy_descent_fpr}
\end{figure}

\section{Model Repair}
\label{Sec::ModelRepair}

In this section, we describe how to use counterfactual distributions to repair classifiers that exhibit disparate impact.

\paragraph{Preprocessor}

Given a classifier $h(\bx)$, we aim to mitigate disparate impact by constructing a  \emph{preprocessor} $T:\mathcal{X}\to\mathcal{X}$ that alters the features of the target group. Thus, the \emph{repaired classifier} $\tilde{h}(\bx)$ will operate as:
\begin{align}
\tilde{h}(\bx) = \begin{cases}
h(T(\bx)) & \text{if } s = 0, \\
h(\bx) & \text{otherwise}.
\end{cases}
\end{align}
The preprocessor is a (potentially randomized) mapping that transforms the distribution of samples over the target population into the counterfactual distribution, i.e., given a random variable $X$ drawn from the target population distribution, the distribution of $T(X)$ will approximate counterfactual distribution. 

\paragraph{Optimal Transport} 

We produce the preprocessor by solving an optimal transport problem. To this end, we require the following inputs: 
\begin{itemize}

    \item $\mathcal{D}_0$, which represents the original samples for the target group. We assume $\mathcal{D}_0$ contains $n_0$ samples, of which $m$ are distinct: $\{\bx_1,\cdots,\bx_m\}$.
    
    \item $\tilde{\mathcal{D}}_0$, which represents the samples drawn from the counterfactual distribution (i.e., the data produced via resampling in  Algorithm~\ref{alg:descent}). We assume that $\tilde{\mathcal{D}}_0$ contains $\tilde{n}_0$ samples, of which $\tilde{m}$ are distinct: $\{\tilde{\bx}_1,\cdots,\tilde{\bx}_{\tilde{m}}\}$. 
    
\end{itemize} 
With these samples at hand, we formulate an optimal transport problem of the form:

% MIP
\newcommand{\st}{\textnormal{s.t.}}
\newcommand{\miprange}[3]{{#1}={#2},\cdots,{#3}}
\newcommand{\mipwhat}[1]{\textit{\tiny #1}}

\begin{subequations}
\label{LP::OptimalTransport}
\begin{equationarray}{@{}c@{}r@{\quad}l>{\;}r@{\;}}
\min_{\gamma_{ij} \in \mathbb{R}^+} & \quad \sum_{i=1}^m \sum_{j=1}^{\tilde{m}} C_{ij} \gamma_{ij} & & \mipwhat{} \label{Con::OTObj} \\ 
\st
& \sum_{j=1}^{\tilde{m}} \gamma_{ij} =   p_i &  \miprange{i}{1}{m}  & \mipwhat{} \label{Con::OTP}\\
& \sum_{i=1}^{m} \gamma_{ij} =  q_j &  \miprange{j}{1}{\tilde{m}}.  & \mipwhat{} \label{Con::OTQ}
\end{equationarray}
\end{subequations}

Here, $C_{ij}$ represents the cost of altering the input variables from $\bx_i$ to $\tilde{\bx}_j$ given a user-specified \emph{cost function} that we will discuss shortly; $p$, $q$ are the empirical estimates of $P_0$ and $Q_X$, respectively,
\begin{align*}
    p_i = \frac{1}{n_0}\sum_{\bx\in \mathcal{D}_0} \delta_{\bx_i}(\bx),\quad q_j = \frac{1}{\tilde{n}_0}\sum_{\bx\in \tilde{\mathcal{D}}_0} \delta_{\tilde{\bx}_j}(\bx);
\end{align*}
The optimal transport problem in \eqref{LP::OptimalTransport} is a standard linear program that aims to find a \emph{coupling} of $p$ and $q$, $\gamma$ \citep[see e.g.,][]{peyre2017computational,villani2008optimal}. 
Formally, a coupling is a joint probability distribution with marginal distributions specified by $p$ and $q$. Given the minimal-cost coupling $\gamma^*$, one can construct a (randomized) preprocessor $T(\cdot)$ which takes a sample $\bx_i$ and returns an altered sample $\tilde{\bx}_j$ with probability $\gamma^*_{ij}/p_i$. 

We note that the linear programming formulation in \eqref{LP::OptimalTransport} is designed for settings with discrete input distributions. In settings when the distributions $P_0$ and $Q_X$ are continuous, an analogous optimal transport problem can be formulated and solved with other approaches~\citep[see e.g.,][]{benamou2000computational,angenent2003minimizing}.

\paragraph{Choice of Cost Function}

The cost function $C_{ij}$ controls how samples of the target group are perturbed. By default, one could use a standard distance metric such as the $L_2$-norm \citep[e.g., ][]{feldman2015certifying,johndrow2017algorithm,del2018obtaining}. However, one could also consider additional criteria to fine-tune the mapping specified by $T(\cdot).$ For example, one can specify a cost function that avoids specific kinds of change by setting the value of $C_{ij}$ to a large constant so as to penalize undesirable mappings \citep[e.g., a mapping that would alter immutable attributes such as marital status][]{ustun2019actionable}.

\paragraph{Customization via Constraints}

Users can also fine-tune the behavior of the preprocessor by adding custom constraints to the feasible region of optimal transport problems as in \eqref{LP::OptimalTransport}. For example, one can impose constraints on \emph{individual fairness} to ensure that the repaired classifier will ``treat similar individuals similarly'' \citep[see e.g.,][]{dwork2012fairness}. This behavior could be induced by including constraints of the form:
\begin{align*}
    \frac{1}{2}\sum_{j=1}^{\tilde{m}}\left|\frac{\gamma_{ij}}{p_i} - \frac{\gamma_{lj}}{p_l}\right| \leq \mathsf{d}(\bx_i,\bx_l) \quad \textrm{for all}\ i, l \in [m].
\end{align*}
Here, the LHS is the total-variation distance \citep{cover2012elements} between the distributions of $T(\bx_i)$ and $T(\bx_l)$, and  $\mathsf{d}(\cdot,\cdot): \mathcal{X}\times\mathcal{X}\to \Reals^{+}$ is a distance metric that reflects the similarity between samples.

\section{Experiments}
\label{Sec::Experiments}

In this section, we demonstrate how counterfactual distributions can be used to avoid disparate impact for classifiers on real-world datasets. We include all datasets and scripts to reproduce our results at \gitrepo{}.

\paragraph{Setup}

We aim to recover counterfactual distributions for different disparity metrics in Table \ref{tabel:ExpDiscMet}. To this end, we consider processed versions of the \texttt{adult} dataset \citep{bache2013uci} and the ProPublica \texttt{compas} dataset \citep{angwin2016machine}.

For each dataset, we use: 
\begin{itemize}
    \item 30\% of samples to train a classifier $h(\bx)$ to repair; 
    \item 50\% of samples to recover a counterfactual distribution via Algorithm \ref{alg:descent};
    \item 20\% of samples as a hold-out set to evaluate the performance of the repaired model.
\end{itemize}

We use $\ell_2$-logistic regression to train a classifier $h(\bx)$ as well as the classifiers $\hat{y}_0(\bx)$ and $\hat{s}(\bx)$ that we use to estimate the influence functions in Algorithm~\ref{alg:descent}. We tune the parameters and estimate the performance of each classifiers using a standard 10-fold CV setup.

Our setup assumes that the data used to train and repair the model are drawn from the same distribution, which may not be the case in settings with dataset shift. Our setup also differs from real-world settings in that we use 70\% of the samples in each dataset to estimate the counterfactual distribution. In practice, however, we would use all available samples since we would be given the classifier to repair.

\paragraph{Discussion}

In Table \ref{Table::Preprocessing}, we show the effectiveness of preprocessors built to mitigate different kinds of disparity for the classifiers trained on the \texttt{adult} and \texttt{compas} datasets. 

We build each preprocessor as follows. We first resample data from the target population according to Algorithm~\ref{alg:descent}. This outputs a dataset of samples drawn from the counterfactual distribution. Next, we use the resampled dataset to produce an empirical estimate of the counterfactual distribution $Q_X$. This distribution is then used to obtain the preprocessor by solving a version of \eqref{LP::OptimalTransport} with the cost function $C_{ij}= \|\bx_i-\tilde{\bx}_j\|_2^2$.

As shown, the approach reduces disparate impact in the target group, while having a minor effect on test accuracy at various decision points across the full ROC curve. Counterfactual distributions provide a way to scrutinize this mapping in greater detail. As shown in Table \ref{Table::CounterFactual_adult}, one can visualize the differences between the observed distribution and counterfactual distribution to understand how the input variable distributions are altered to reduce disparity. 

This kind of constrastive analysis may be helpful in understanding the factors that produce performance disparities in the first place. For example, the differences between the observed distribution $\PXa$ and the counterfactual distribution $Q_X$ could be used to identify prototypical samples \citep[see e.g,][]{bien2011prototype,kim2016examples}, or to score features in terms of their ability to produce disparities in the deployment population \cite{datta2016algorithmic,datta2017proxy,adler2018auditing}.

\newcommand{\DA}[1]{\textrm{DA}_{#1}}

\begin{table*}[t]
\small\centering
\resizebox{0.9\textwidth}{!}{
\renewcommand{\arraystretch}{1.1}
\begin{tabular}{lllcccrrcc}
\toprule
&
&
& \multicolumn{3}{c}{\textsc{Original Model}} 
& \multicolumn{2}{c}{\textsc{Repaired Model}} 
& \multicolumn{2}{c}{\textsc{Target Group AUC}} \\
\cmidrule(lr){4-6} \cmidrule(lr){7-8} \cmidrule(lr){9-10}
%\toprule
%\textsc{Observed Discrimination} & \textsc{With Transformer} & \textsc{Performance Loss over Target Group}\\
%\toprule
\textsc{Dataset} & 
\textsc{Metric} & 
\sccell{l}{Target\\Group} & 
\sccell{c}{Baseline\\Group} & 
\sccell{c}{Target\\Group} & 
\sccell{c}{\textbf{Disc.}\\\textbf{Gap}} & 
\sccell{c}{Target\\Group} & 
\sccell{c}{\textbf{Disc.}\\\textbf{Gap}} & 
\sccell{c}{Before\\Repair} & 
\sccell{c}{After\\Repair} \\ 

\toprule

\texttt{adult} & SP & Female & 0.696 & 0.874 & \textbf{0.178} & 0.688 & \textbf{-0.007} & 0.895 & 0.758

\\ \midrule

\texttt{adult} & FNR & Female & 0.478 & 0.639 & \textbf{0.161} & 0.483 & \textbf{0.004} & 0.895 & 0.880

\\ \midrule

\texttt{adult} & FPR & Male & 0.021 & 0.119 & \textbf{0.098} & 0.023 & \textbf{0.002} & 0.829 & 0.714

\\ \toprule

\texttt{compas} & SP &  White & 0.514 & 0.594  & \textbf{0.079} & 0.533  & \textbf{0.018} & 0.704  &  0.667

\\\midrule

\texttt{compas} & FNR & White & 0.350 & 0.487 & \textbf{0.137} & 0.439 & \textbf{0.088} & 0.704 & 0.699

\\\midrule

\texttt{compas} & FPR & Non-white & 0.190 & 0.278 & \textbf{0.087} & 0.160 & \textbf{-0.029} & 0.732 & 0.680

\\ \bottomrule 
\end{tabular}
}
\caption{Change in disparate impact for classification models for \texttt{adult} and \texttt{compas} when paired with a randomized preprocessor built to mitigate different kinds of disparity. Each row shows the value of a specific performance metric for the classifier over the target and baseline groups (e.g., SP, FNR, and FPR). The target group is defined as the group that attains the less favorable value of the performance metric. The preprocessor aims to reduce to difference in performance metric by randomly perturbing the input variables for individuals in the target group. We also include AUC to show the change in performance due to the randomized preprocessor. All values are computed using a hold-out sample that is not used to train the model or build the preprocessor.}
\label{Table::Preprocessing}
\end{table*}
\begin{table*}[h]
\centering
\resizebox{0.5\textwidth}{!}{
\begin{tabular}{lrr@{\hspace{0.4cm}}rrr}
\toprule
& \multicolumn{2}{c}{\textsc{Observed}}

& \multicolumn{3}{c}{\textsc{Counterfactual}} \\
\cmidrule(lr){2-3} 
\cmidrule(lr){4-6}
{} &  
Female &  
Male &  
\cell{c}{SP\\Female}  &  
\cell{c}{FNR\\Female} & 
\cell{c}{FPR\\Male}\\ 
\midrule
Married                      &      18\% &    63\% &    39\%  & 23\% &   54\% \\
Immigrant                    &      10\% &    11\% &    11\% & 11\% &   12\%  \\
HighestDegree\_is\_HS        &      32\% &    32\% &    24\% & 28\% &   37\%  \\
HighestDegree\_is\_AS        &       7\% &     8\% &     9\% &   9\%  &   6\%  \\
HighestDegree\_is\_BS        &      15\% &    18\% &    21\% & 17\% &   13\%  \\
HighestDegree\_is\_MSorPhD   &       6\% &     7\% &    13\% &   8\%  &   5\% \\
AnyCapitalLoss               &       3\% &     5\% &     8\% &   5\%  &   4\%  \\
Age $\leq$ 30                 &      39\% &    29\% &    29\% & 38\% &   35\% \\
WorkHrsPerWeek$<$40       &      38\% &    17\% &    33\% & 37\% &   19\% \\
JobType\_is\_WhiteCollar     &      34\% &    19\% &    36\% & 35\% &   15\% \\
JobType\_is\_BlueCollar      &       5\% &    34\% &     4\% &  5\%  &   39\% \\
JobType\_is\_Specialized     &      23\% &    21\% &    29\% & 23\% &   20\% \\
JobType\_is\_ArmedOrProtective &     1\% &     2\% &     1\% &   1\%  &   3\% \\
Industry\_is\_Private        &      73\% &    69\% &    64\% & 69\% &   70\% \\
Industry\_is\_Government     &      15\% &    12\% &    22\% & 17\% &   12\% \\
Industry\_is\_SelfEmployed   &       5\% &    15\% &     8\% & 6\%  &   13\% \\
\bottomrule
\end{tabular}
}
\caption{Counterfactual distributions produced using Algorithm \ref{alg:descent} for a classifier on \texttt{adult}. We observe that different metrics produce different counterfactual distributions. By comparing the distribution of the target group with the counterfactual distribution, we can evaluate how the repaired classifier will perturb their features to reduce disparity.}
\label{Table::CounterFactual_adult}
\end{table*}
\section{Conclusion}
\label{Sec::Conclusion}

We have introduced a new distributional paradigm to understand and mitigate disparate impact. Our framework is based on counterfactual distributions, which can be efficiently computed given a fixed model and data from a population of interest. The tools in this work apply to binary classification models. However, our approach can be extended to handle other kinds of supervised learning models.

\paragraph{Limitations}

Our approach requires collecting data on sensitive attribute, which may infringe privacy \citep[though this is difficult to avoid as discussed in][]{vzliobaite2016using}. Our approach also aims to mitigate performance using a randomized preprocessor. While randomization is a common technique to reduce disparate impact in the literature \citep[see e.g.,][]{hardt2016equality,agarwal2018reductions}, it may not be practical in applications such as loan approval since an applicant could achieve a different predicted outcome by applying multiple times. Some effects of randomization can be mitigated by heuristic strategies. Given that we have considered counterfactual distributions that improve the performance for the target group, however, our approach does have a benefit in that  randomization will only apply to individuals in the target group and only be applied in a way that will improve their outcomes.

\clearpage
\section*{Acknowledgements}
This material is based upon work supported by the National Science Foundation under Grant No. CCF-18-45852.
{
\small
\bibliographystyle{icml2019}
\bibliography{ctf_references}
}

\clearpage
\onecolumn
\appendix
\onecolumn
\section{Omitted Proofs}
\label{sec::App_proofs}

\subsection{Factorization of Joint Distribution}
\label{append::fact_metrics}

In this section, we show that the disparity metrics in Table~\ref{tabel:ExpDiscMet} can be expressed in terms of $\PXa$ when $P_{\hat{Y}|X}$, $P_{Y|X,S}$, $\PXb$, and $P_S$ are given.

We observe that since our classifier is fixed, the joint distribution $P_{S,X,Y,\hat{Y}}$ is characterized by the graphical model shown in Figure~\ref{Fig::GM}. Accordingly, we can express $P_{S,X,Y,\hat{Y}}$ as:
\begin{align}
P_{S,X,Y,\hat{Y}}=P_{\hat{Y}|X}P_{Y|X,S}P_SP_{X|S}.
\end{align}
Note that $h(\bx)=P_{\hat{Y}|X}(1|\bx)$.
\begin{figure}[b]
  \centering
  %\hspace{-.3in}
  \includegraphics[width=0.25\textwidth]{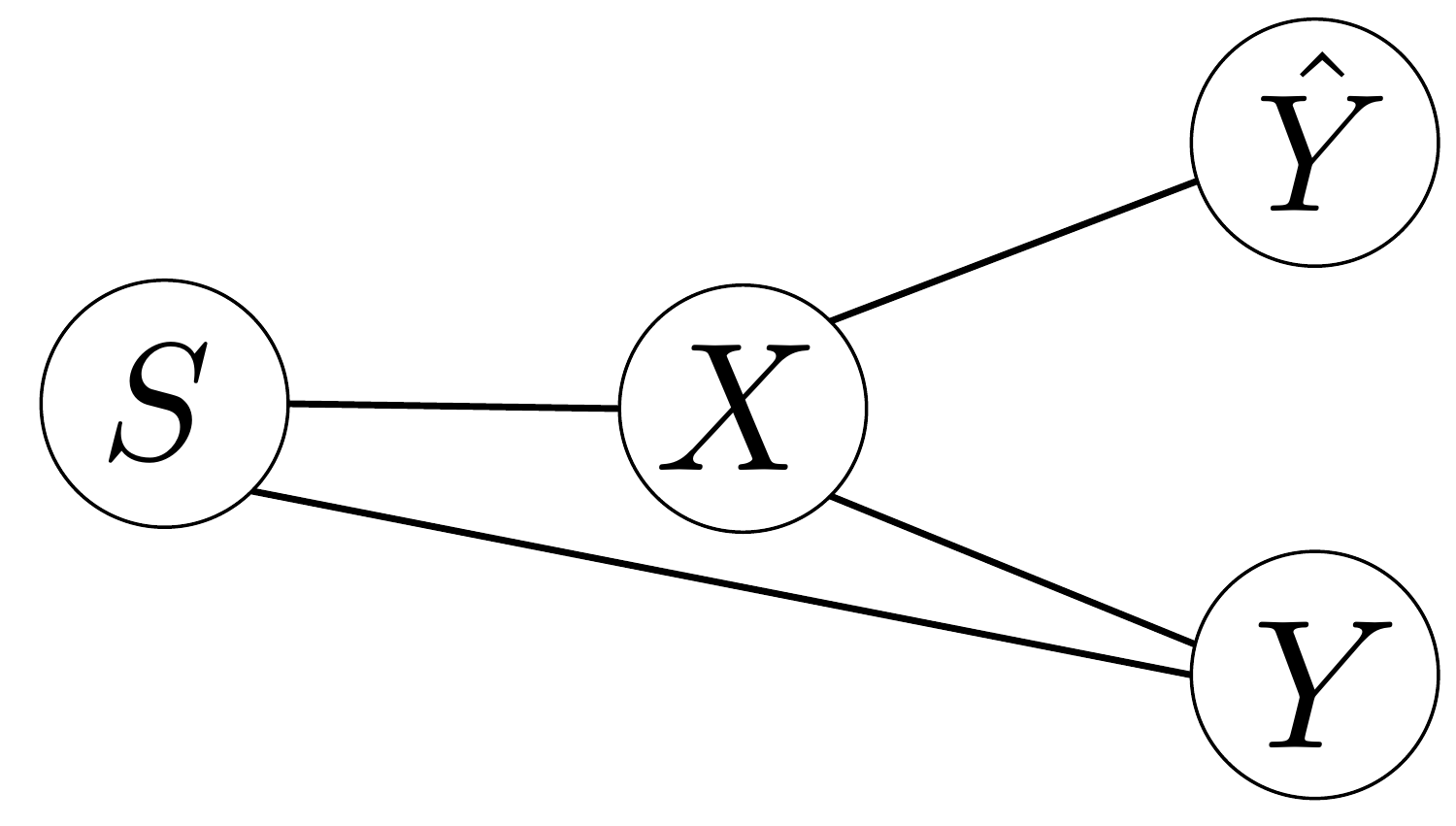}
  \caption{\small{Graphical model of the framework.}}
  \label{Fig::GM}
\end{figure}
In what follows, we use these observations to express each of the disparity metrics in Table~1 as $\mathsf{M}(\PXa)$ (i.e., a function of $\PXa$).

\begin{enumerate}
    \item \textbf{DA.}
    \begin{align}
        \KL(\PYhata\|\PYhatb)+\lambda\KL(\PXa\|\PXb) = \KL(\W\circ \PXa\|\W\circ\PXb)+\lambda\KL(\PXa\|\PXb).
    \end{align}
    \item \textbf{SP.}
    \begin{align}
        \Pr(\hat{Y}=0|S=0)-\Pr(\hat{Y}=0|S=1)
        &=\EE{(1-h(X))|S=0}-\EE{(1-h(X))|S=1}\nonumber\\
        &=-\sum_{\bx\in\mathcal{X}}h(\bx)\PXa(\bx)+\sum_{\bx\in\mathcal{X}}h(\bx)\PXb(\bx).
    \end{align}
    \item \textbf{FDR.}
    \begin{align}
        &\Pr(Y=0|\hat{Y}=1,S=0)-\Pr(Y=0|\hat{Y}=1,S=1)\nonumber\\
        &=\frac{\Pr(Y=0,\hat{Y}=1,S=0)}{\Pr(\hat{Y}=1,S=0)}-\frac{\Pr(Y=0,\hat{Y}=1,S=1)}{\Pr(\hat{Y}=1,S=1)}\nonumber\\
        &=\frac{\sum_{\bx\in\mathcal{X}}P_{\hat{Y}|X}(1|\bx)P_{Y|X,S=0}(0|\bx)\PXa(\bx)}{\sum_{\bx\in\mathcal{X}}P_{\hat{Y}|X}(1|\bx)\PXa(\bx)}-\frac{\sum_{\bx\in\mathcal{X}}P_{\hat{Y}|X}(1|\bx)P_{Y|X,S=1}(0|\bx)\PXb(\bx)}{\sum_{\bx\in\mathcal{X}}P_{\hat{Y}|X}(1|\bx)\PXb(\bx)}.
    \end{align}
    \item \textbf{FNR.}
    \begin{align}
        &\Pr(\hat{Y}=0|Y=1,S=0)-\Pr(\hat{Y}=0|Y=1,S=1)\nonumber\\
        &=\frac{\sum_{\bx\in\mathcal{X}}P_{\hat{Y}|X}(0|\bx)P_{Y|X,S=0}(1|\bx)\PXa(\bx)}{\sum_{\bx\in\mathcal{X}}P_{Y|X,S=0}(1|\bx)\PXa(\bx)}-\frac{\sum_{\bx\in\mathcal{X}}P_{\hat{Y}|X}(0|\bx)P_{Y|X,S=1}(1|\bx)\PXb(\bx)}{\sum_{\bx\in\mathcal{X}}P_{Y|X,S=1}(1|\bx)\PXb(\bx)}.
    \end{align}
    \item \textbf{FPR.}
    \begin{align}
        &\Pr(\hat{Y}=1|Y=0,S=0)-\Pr(\hat{Y}=1|Y=0,S=1)\nonumber\\
        &=\frac{\sum_{\bx\in\mathcal{X}}P_{\hat{Y}|X}(1|\bx)P_{Y|X,S=0}(0|\bx)\PXa(\bx)}{\sum_{\bx\in\mathcal{X}}P_{Y|X,S=0}(0|\bx)\PXa(\bx)}-\frac{\sum_{\bx\in\mathcal{X}}P_{\hat{Y}|X}(1|\bx)P_{Y|X,S=1}(0|\bx)\PXb(\bx)}{\sum_{\bx\in\mathcal{X}}P_{Y|X,S=1}(0|\bx)\PXb(\bx)}.
    \end{align}
\end{enumerate}

\subsection{Example of Counterfactual Distributions}
\label{append:ex_ctf}

We show that the counterfactual distributions are not always unique.
\begin{example}
\label{ex::multiple_ctf}
We use SP as a disparity metric and set $X|S=0 \sim \mathsf{Bernoulli}(0.1)$, $X|S=1 \sim \mathsf{Bernoulli}(0.2)$. The classifier is chosen as $h(0)=h(1)=0.2$. In this case, any Bernoulli distribution, including $\PXa$ and $\PXb$, over $\{0,1\}$ is a counterfactual distribution.
\end{example}

\subsection{Proof of Proposition~\ref{prop::ctf_posi_dist_diff}}
\begin{proof}
First, the counterfactual distributions under DA or SP always achieve zero of the disparity metric. Hence, $\mathsf{M}(Q_X)>0$ happens only if the disparity metric is neither DA nor SP. We assume that $P_{Y|X,S=0}=P_{Y|X,S=1}$ and $\mathsf{M}(Q_X)>0$. In particular, $|\mathsf{M}(\PXb)|\geq \mathsf{M}(Q_X) >0$. Note that the disparity metrics in Table~\ref{tabel:ExpDiscMet} except DA are the form of the discrepancies of performance metrics between two groups. Here the performance metrics for each group only depend on $P_{Y|X,S=i}$, $P_{X|S=i}$, and $\W$. If we assume that $P_{Y|X,S=0}=P_{Y|X,S=1}$ and set the distribution of target group as $\PXb$, then the performance metrics achieve the same values for two groups. Hence, $\mathsf{M}(\PXb)=0$ which contradicts the assumption, so $P_{Y|X,S=0}\neq P_{Y|X,S=1}$.
\end{proof}

\subsection{Proof of Proposition~\ref{thm::InfFunc_Corr_Func}}
%~\ref{thm::InfFunc_Corr_Func}
\begin{proof}
First, we define
\begin{align}
    \Delta(f) \defined \lim_{\epsilon \to 0} \frac{\mathsf{M}(\PXp) - \mathsf{M}(\PXa)}{\epsilon},
\end{align}
where $\PXp(\bx)$ is the perturbed distribution defined in \eqref{Eq::Perturb_dist}. Then we prove that
\begin{align*}
    \Delta(f) = \EE{f(X)\psi(X)|S=0}.
\end{align*}
Note that an alternative way \citep[see e.g.,][]{huber2011robust} to define influence functions is in terms of the G\^ateaux derivative:
\begin{align*}
    \sum_{\bx\in\mathcal{X}} \psi(\bx)\PXa(\bx) &= 0,
\end{align*}
and
\begin{align*}
    \lim_{\epsilon\to 0}\frac{1}{\epsilon}\left(\mathsf{M}\left((1-\epsilon)\PXa+\epsilon Q\right)-\mathsf{M}\left(\PXa\right)\right)
    = \sum_{\bx\in\mathcal{X}} \psi(\bx)Q(\bx),\ \forall Q\in \mathcal{P}.
\end{align*}
In particular, we can choose $Q(\bx)=\left(\frac{1}{M_U}f(\bx)+1\right)\PXa(\bx),$
where $M_U \defined \sup\{|f(\bx)| \mid \bx\in\mathcal{X}\}+1$. Then
\begin{align*}
    (1-\epsilon)\PXa(\bx)+\epsilon Q(\bx) = \PXa(\bx) + \frac{\epsilon}{M_U}f(\bx)\PXa(\bx).
\end{align*}
For simplicity, we use $\PXa+\epsilon f\PXa$ and $\PXa + \frac{\epsilon}{M_U}f\PXa$ to represent $\PXa(\bx)+\epsilon f(\bx)\PXa(\bx)$ and $\PXa(\bx) + \frac{\epsilon}{M_U}f(\bx)\PXa(\bx)$, respectively. Then
\begin{align*}
    \Delta(f)
    &=\lim_{\epsilon\to 0}\frac{1}{\epsilon}\left(\mathsf{M}(\PXa+\epsilon f\PXa)
    -\mathsf{M}(\PXa)\right)\\
    &=\lim_{\epsilon\to 0}\frac{M_U}{\epsilon}\left(\mathsf{M}\left(\PXa+\frac{\epsilon}{M_U} f\PXa\right)
    -\mathsf{M}(\PXa)\right)\\
    &=M_U\lim_{\epsilon\to 0}\frac{1}{\epsilon}\left(\mathsf{M}((1-\epsilon)\PXa+\epsilon Q)
    -\mathsf{M}(\PXa)\right)\\
    &=M_U\sum_{\bx\in\mathcal{X}} \psi(\bx)Q(\bx)\\
    &=M_U\sum_{\bx\in\mathcal{X}} \psi(\bx)\left(\frac{1}{M_U}f(\bx)+1\right)
    \PXa(\bx)\\
    &=\sum_{\bx\in\mathcal{X}} \psi(\bx)f(\bx)\PXa(\bx)\\
    &=\EE{f(X)\psi(X)|S=0}.
\end{align*}
Following from Cauchy-Schwarz inequality,
\begin{align*}
    \EE{f(X)\psi(X)|S=0}
    \geq -\sqrt{\EE{f(X)^2|S=0}}
    \sqrt{\EE{\psi(X)^2|S=0}}=-\sqrt{\EE{\psi(X)^2|S=0}}.
\end{align*}
Here the equality can be achieved by choosing
\begin{align*}
    f(\bx) = \frac{-\psi(\bx)}{\sqrt{\EE{\psi(X)^2|S=0}}}.
\end{align*}
\end{proof}

\subsection{Proof of Proposition~\ref{thm::Linearility_Inf_Cor}}
\begin{proof}
When the disparity metric is a linear combination of $K$ different disparity metrics: 
\begin{align*}
    \M = \sum_{i=1}^K\lambda_i\mathsf{M}_i(\PXa),
\end{align*}
the influence function, following from Definition~\ref{defn::inf_func}, is
\begin{align}
    \psi(\bx)
    &= \lim_{\epsilon\to 0}\frac{\mathsf{M}\left((1-\epsilon)\PXa+\epsilon \delta_{\bx}\right) -\mathsf{M}(\PXa)}{\epsilon}\\
    &= \sum_{i=1}^K\lambda_i \lim_{\epsilon\to 0}\frac{\mathsf{M}_i\left((1-\epsilon)\PXa+\epsilon \delta_{\bx}\right) -\mathsf{M}_i(\PXa)}{\epsilon}\\
    &= \sum_{i=1}^K\lambda_i \psi_i(\bx).
\end{align}

\end{proof}

\subsection{Proofs of Proposition~\ref{prop::inf_func_exp}}
\label{append::closed_form_inf_func}

We prove the closed-form expressions of influence functions provided in Proposition~\ref{prop::inf_func_exp} in this section. Again, we view the classifier $h(\bx)$ as a conditional distribution $P_{\hat{Y}|X}(1|\bx)$.

\begin{proof}
\textbf{Influence function for SP.}
Recall that 
\begin{align*}
    \Pr(\hat{Y}=0|S=0)
    =1-\sum_{\bx\in\mathcal{X}}h(\bx)\PXa(\bx).
\end{align*}
When we perturb the distribution $\PXa$, the classifier $h(\bx)$ and $\Pr(\hat{Y}=1|S=1)$ do not change. Therefore,
\begin{align*}
    \psi(\bx)
    &= \lim_{\epsilon\to0}
    -\frac{1}{\epsilon}\left(\sum_{\bx'\in \mathcal{X}} h(\bx')((1-\epsilon)\PXa(\bx')+\epsilon\delta_{\bx}(\bx'))
    -\sum_{\bx'\in \mathcal{X}} h(\bx')\PXa(\bx')\right)\\
    &=-h(\bx)+\Pr(\hat{Y}=1|S=0).
\end{align*}

\textbf{Influence function for FNR.} Next, we compute the influence function of FNR. Similar analysis holds for FPR and FDR. Due to the factorization of the joint distribution (see Appendix~\ref{append::fact_metrics}), we have
\begin{align*}
    \Pr(\hat{Y}=0|Y=1,S=0)
    =\frac{\sum_{\bx'\in\mathcal{X}}\W(0|\bx')P_{Y|X,S=0}(1|\bx')\PXa(\bx')}{\sum_{\bx'\in\mathcal{X}}P_{Y|X,S=0}(1|\bx')\PXa(\bx')}.
\end{align*}
We denote $r_1(\bx) \defined \W(0|\bx)P_{Y|X,S=0}(1|\bx)$ and $r_2(\bx) \defined P_{Y|X,S=0}(1|\bx)$. Then
\begin{align*}
    \Pr(\hat{Y}=0|Y=1,S=0) 
    =\frac{\sum_{\bx'\in\mathcal{X}}r_1(\bx')\PXa(\bx')}{\sum_{\bx'\in\mathcal{X}}r_2(\bx')\PXa(\bx')}
    =\frac{\EE{r_1(X)|S=0}}{\EE{r_2(X)|S=0}},
\end{align*}
which implies
\begin{align*}
    &\mathsf{M}((1-\epsilon)\PXa+\epsilon \delta_{\bx})\\
    &= \frac{\sum_{\bx'\in\mathcal{X}}r_1(\bx')((1-\epsilon)\PXa(\bx')+\epsilon\delta_{\bx}(\bx'))}{\sum_{\bx'\in\mathcal{X}}r_2(\bx')((1-\epsilon)\PXa(\bx')+\epsilon\delta_{\bx}(\bx'))}-\Pr(\hat{Y}=0|Y=1,S=1)\\
    &= \frac{\EE{r_1(X)|S=0}+\epsilon\left( r_1(\bx)-\EE{r_1(X)|S=0}\right)}{\EE{r_2(X)|S=0}+\epsilon \left(r_2(\bx)-\EE{r_2(X)|S=0}\right)}-\Pr(\hat{Y}=0|Y=1,S=1).
\end{align*}
Therefore,
\begin{align*}
    \psi(\bx)
    &= \lim_{\epsilon\to 0} \frac{1}{\epsilon}\left(\mathsf{M}((1-\epsilon)\PXa+\epsilon \delta_{\bx})
    -\mathsf{M}(\PXa)\right)\\
    &=\frac{\EE{r_2(X)|S=0} r_1(\bx)-\EE{r_1(X)|S=0}r_2(\bx)}{\EE{r_2(X)|S=0}^2}\\
    &=\frac{\Pr(Y=1|S=0) r_1(\bx)-\Pr(\hat{Y}=0,Y=1|S=0)r_2(\bx)}{\Pr(Y=1|S=0)^2}\\
    &=\frac{\W(0|\bx)P_{Y|X,S=0}(1|\bx)-\Pr(\hat{Y}=0|Y=1,S=0)P_{Y|X,S=0}(1|\bx)}{\Pr(Y=1|S=0)}.
\end{align*}

\textbf{Influence function for DA.} We start with computing $\KL((1-\epsilon)\PXa+\epsilon\delta_{\bx}\|\PXb)$:
\begin{align*}
    \KL((1-\epsilon)\PXa+\epsilon\delta_{\bx}\|\PXb)
    &=\sum_{\bx'\in\mathcal{X}}((1-\epsilon)\PXa(\bx')+\epsilon \delta_{\bx}(\bx')) \log\frac{(1-\epsilon)\PXa(\bx')+\epsilon \delta_{\bx}(\bx')}{\PXb(\bx')}\\
    &=\sum_{\bx'\in\mathcal{X}}(\PXa(\bx')+\epsilon (\delta_{\bx}(\bx')-\PXa(\bx')))\\ &\quad\quad\quad\times\left(\log\frac{\PXa(\bx')}{\PXb(\bx')} 
    +\log\left(1+\frac{\epsilon(\delta_{\bx}(\bx')-\PXa(\bx'))}{\PXa(\bx')}\right)\right)\\
    &=\sum_{\bx'\in\mathcal{X}}(\PXa(\bx')+\epsilon (\delta_{\bx}(\bx')-\PXa(\bx')))\\ &\quad\quad\quad\times\left(\log\frac{\PXa(\bx')}{\PXb(\bx')}
    +\epsilon\frac{\delta_{\bx}(\bx')-\PXa(\bx')}{\PXa(\bx')}+O(\epsilon^2)\right)\\
    &=\KL(\PXa\|\PXb)
    +\epsilon\sum_{\bx'\in\mathcal{X}} (\delta_{\bx}(\bx')-\PXa(\bx'))\log\frac{\PXa(\bx')}{\PXb(\bx')}
    +O(\epsilon^2)\\
    &=\KL(\PXa\|\PXb)
    +\epsilon\left(\log\frac{\PXa(\bx)}{\PXb(\bx)} - \EE{\log\frac{\PXa(X)}{\PXb(X)}\Bigg|S=0}\right)
    +O(\epsilon^2).
\end{align*}
Hence,
\begin{align}
    \lim_{\epsilon\to0}
    \frac{1}{\epsilon}\left(\KL((1-\epsilon)\PXa+\epsilon\delta_{\bx}\|\PXb)-\KL(\PXa\|\PXb)\right)
    =\log\frac{\PXa(\bx)}{\PXb(\bx)} - \EE{\log\frac{\PXa(X)}{\PXb(X)}\Bigg|S=0}\label{equa::inf_rpp_a}.
\end{align}
Similarly, we have
\begin{align}
    &\lim_{\epsilon\to0}
    \frac{1}{\epsilon}\left(\KL((1-\epsilon)\PYhata+\epsilon\W\circ\delta_{\bx}\|\PYhatb)-\KL(\PYhata\|\PYhatb)\right) \nonumber\\
    &=\sum_{y\in\{0,1\}} ((\W\circ \delta_{\bx})(y)-\PYhata(y))\log\frac{\PYhata(y)}{\PYhatb(y)}\nonumber\\
    &=\sum_{y\in\{0,1\}}\log\frac{\PYhata(y)}{\PYhatb(y)}\W(y|\bx)
    -\EE{\log\frac{\PYhata(\hat{Y})}{\PYhatb(\hat{Y})}\Bigg|S=0} \label{equa::inf_rpp_b}.
\end{align}
Combining \eqref{equa::inf_rpp_a} with \eqref{equa::inf_rpp_b}, we have
\begin{align*}
    \psi(\bx)
    &= \sum_{y\in\{0,1\}}\log\frac{\PYhata(y)}{\PYhatb(y)}\W(y|\bx)
    -\EE{\log\frac{\PYhata(\hat{Y})}{\PYhatb(\hat{Y})}\Bigg|S=0}\\
    &\quad +\lambda\left(\log\frac{\PXa(\bx)}{\PXb(\bx)} - \EE{\log\frac{\PXa(X)}{\PXb(X)}\Bigg|S=0}\right).
\end{align*}
Note that
\begin{align*}
    \log\frac{\PXa(\bx)}{\PXb(\bx)} 
    =\log\frac{P_{X,S}(\bx,0)}{P_{X,S}(\bx,1)}+\log\frac{P_S(1)}{P_S(0)}
    =\log\frac{P_{S|X}(0|\bx)}{P_{S|X}(1|\bx)}+\log\frac{P_S(1)}{P_S(0)}.
\end{align*}
Hence,
\begin{align*}
    \log\frac{\PXa(\bx)}{\PXb(\bx)} - \EE{\log\frac{\PXa(X)}{\PXb(X)}\Bigg|S=0}
    &=\log\frac{P_{S|X}(0|\bx)}{P_{S|X}(1|\bx)}-\EE{\log\frac{P_{S|X}(0|X)}{P_{S|X}(1|X)}\Bigg|S=0}\\
    &=\log\frac{1-P_{S|X}(1|\bx)}{P_{S|X}(1|\bx)}-\EE{\log\frac{1-P_{S|X}(1|X)}{P_{S|X}(1|X)}\Bigg|S=0}.
\end{align*}
Next,
\begin{align*}
    &\sum_{y\in\{0,1\}}\log\frac{\PYhata(y)}{\PYhatb(y)}\W(y|\bx)- \EE{\log\frac{\PYhata(\hat{Y})}{\PYhatb(\hat{Y})}\Bigg|S=0}\\
    &=\log\frac{\PYhata(1)}{\PYhatb(1)}\W(1|\bx)+\log\frac{\PYhata(0)}{\PYhatb(0)}(1-\W(1|\bx))\\
    &\quad- \EE{\log\frac{\PYhata(\hat{Y})}{\PYhatb(\hat{Y})}\Bigg|S=0}\\
    %&=\log\frac{\PYhata(1)\PYhatb(0)}{\PYhatb(1)\PYhata(0)}\W(1|\bx)\\
    %&\quad +\log\frac{\PYhata(0)}{\PYhatb(0)}-\EE{\log\frac{\PYhata(\hat{Y})}{\PYhatb(\hat{Y})}\Bigg|S=0}\\
    &=\log\frac{\PYhata(1)\PYhatb(0)}{\PYhatb(1)\PYhata(0)}\W(1|\bx)\\
    &\quad +\log\frac{\PYhata(0)}{\PYhatb(0)}-\log\frac{\PYhata(0)}{\PYhatb(0)}\PYhata(0)-\log\frac{\PYhata(1)}{\PYhatb(1)}\PYhata(1)\\
    &=\left(\log\frac{\PYhata(1)\PYhatb(0)}{\PYhatb(1)\PYhata(0)}\right)\W(1|\bx)-\left(\log\frac{\PYhata(1)\PYhatb(0)}{\PYhatb(1)\PYhata(0)}\right)\PYhata(1).
\end{align*}
Therefore, we have
\begin{align*}
    \psi(\bx)
    &=\left(\log\frac{\PYhata(1)\PYhatb(0)}{\PYhatb(1)\PYhata(0)}\right)\left(\W(1|\bx)-\PYhata(1)\right)\\
    &\quad +\lambda\left(\log\frac{1-P_{S|X}(1|\bx)}{P_{S|X}(1|\bx)}-\EE{\log\frac{1-P_{S|X}(1|X)}{P_{S|X}(1|X)}\Bigg|S=0}\right).
\end{align*}
\end{proof}

\subsection{Convergence}
\label{appendix_convergence}

When influence functions are estimated from data, they are subject to estimation error. Next, we present a probabilistic upper bound of the estimation error in terms of the number of samples and the size of the support set.
\begin{prop}
\label{prop::generalizationthm_Cal}
If $\hat{s}(\bx)$ and $\hat{y}_0(\bx)$ are the empirical conditional distributions obtained from $n$ i.i.d. samples, then, with probability at least $1-\beta$,
\begin{align}
\label{equa::GeneralizedBound}
    \left\|\widehat{\psi}(\bx)-\psi(\bx)\right\|_1
    \leq O\left(\sqrt{n^{-1}\left(|\mathcal{X}|-\log\beta\right)}\right).
\end{align}
Here, $\|f(\bx)-g(\bx)\|_1 \defined \EE{|f(X)-g(X)||S=0}$ denotes the $\ell_1$-norm. 
\end{prop}

The proof of Proposition~\ref{prop::generalizationthm_Cal} relies on the following lemma.
\begin{lem}
\label{lem::gen_bound}
Let $\widehat{\psi}(\bx)$ and $\psi(\bx)$ be the estimated influence function and the true influence function, respectively. If the given disparity metric is $\textrm{DA}_{\lambda}$,
\begingroup
\small
\begin{align*}
\left\|\widehat{\psi}(\bx)-\psi(\bx)\right\|_p \leq
O\left(\left\|\widehat{P}_{S|X}(1|\bx)-P_{S|X}(1|\bx)\right\|_p\right).
\end{align*}
\endgroup
For all other disparity metrics in Table~\ref{tabel:ExpDiscMet},
\begingroup
\small
\begin{align*}
\left\|\widehat{\psi}(\bx)-\psi(\bx)\right\|_p 
\leq O\left(\left\|\widehat{P}_{Y|X,S=0}(1|\bx)-P_{Y|X,S=0}(1|\bx)\right\|_p\right).
\end{align*}
\endgroup
Here, $\|f(\bx)-g(\bx)\|_p \defined \left(\EE{|f(X)-g(X)|^p|S=0}\right)^{1/p}$ denotes the $\ell_p$-norm for $p\geq 1$. 
\end{lem}
\begin{proof}
We denote $\widehat{P}$ and $\widehat{\Pr}$ as estimated probability distribution and probability, respectively. Then we assume that
\begin{align}
    \left\|\widehat{P}_{X|S=0}-P_{X|S=0}\right\|_p &\lesssim \left\|\widehat{P}_{Y|X,S=0}(1|\bx)-P_{Y|X,S=0}(1|\bx)\right\|_p;\label{eq::assume_marg_leq_chan}\\
    \left|\widehat{\Pr}(Y=1|S=0)-\Pr(Y=1|S=0)\right| &\lesssim \left\|\widehat{P}_{Y|X,S=0}(1|\bx)-P_{Y|X,S=0}(1|\bx)\right\|_p;\label{eq::assume_bern_1_leq_chan}\\
    \left|\widehat{\Pr}(\hat{Y}=0|Y=1,S=0)-\Pr(\hat{Y}=0|Y=1,S=0)\right|
    &\lesssim \left\|\widehat{P}_{Y|X,S=0}(1|\bx)-P_{Y|X,S=0}(1|\bx)\right\|_p,\label{eq::assume_bern_2_leq_chan}
\end{align}
where $\left\|\widehat{P}_{X|S=0}-P_{X|S=0}\right\|_p \defined \left(\sum_{\bx\in\mathcal{X}}\left|\widehat{P}_{X|S=0}(\bx)-P_{X|S=0}(\bx)\right|^{p}\right)^{1/p}$.
We make similar assumptions for $\widehat{P}_{S|X}(1|\bx)$ (i.e., the $\ell_p$ distance between $\widehat{P}_{S|X}(1|\bx)$ and $P_{S|X}(1|\bx)$ upper bounds the left-hand side of \eqref{eq::assume_marg_leq_chan}, \eqref{eq::assume_bern_1_leq_chan}, \eqref{eq::assume_bern_2_leq_chan}). These assumptions are reasonable in practice since estimating conditional distribution is usually harder than estimating marginal distribution which is harder than estimating the distribution of Bernoulli random variable.
\begin{enumerate}
    \item \textbf{SP.} The influence function under SP is
\begin{align*}
    \psi(\bx) 
    &=-h(\bx)+\Pr(\hat{Y}=1|S=0).
\end{align*}
In order to compute the influence function under SP, we only need to estimate $\Pr(\hat{Y}=1|S=0)$ since the classifier $h(\bx)$ is given. Estimating the distribution of a Bernoulli random variable is more reliable than estimating the conditional distribution so 
\begin{align*}
    &\|\widehat{\psi}(\bx)-\psi(\bx)\|_p 
    \lesssim \left\|P_{Y|X,S=0}(1|\bx)-\widehat{P}_{Y|X,S=0}(1|\bx)\right\|_p.
\end{align*}

\item \textbf{Class-Based Error Metrics.} Next, we present a proof of the generalization bound for FNR. Similar proofs hold for other class-based error metrics such as FDR and FPR.

The influence function under FNR is
\begin{align*}
    \psi(\bx)=\frac{\EE{r_2(X)|S=0}r_1(\bx)-\EE{r_1(X)|S=0}r_2(\bx)}{\Pr(Y=1|S=0)^2},
\end{align*}
where $r_1(\bx) = \W(0|\bx)P_{Y|X,S=0}(1|\bx)$ and $r_2(\bx) = P_{Y|X,S=0}(1|\bx)$. Note that 
\begin{align*}
    \EE{r_2(X)|S=0}&=\Pr(Y=1|S=0),\\ \EE{r_1(X)|S=0}&=\Pr(\hat{Y}=0,Y=1|S=0).
\end{align*}
Hence, the influence function under FNR has the following equivalent expression.
\begin{align}
    \psi(\bx)
    %&=\frac{\Pr(Y=1|S=0) \W(0|\bx)P_{Y|X,S=0}(1|\bx)-\Pr(\hat{Y}=0,Y=1|S=0)P_{Y|X,S=0}(1|\bx)}{\Pr(Y=1|S=0)^2}\\
    &=\frac{\Pr(Y=1|S=0)\W(0|\bx)-\Pr(\hat{Y}=0,Y=1|S=0)}{\Pr(Y=1|S=0)^2}P_{Y|X,S=0}(1|\bx)\nonumber\\
    &=\frac{\W(0|\bx)-\Pr(\hat{Y}=0|Y=1,S=0)}{\Pr(Y=1|S=0)}P_{Y|X,S=0}(1|\bx).\label{eq::psi_fnr_proof_eq_exp1}
\end{align}
The estimated influence function under FNR is
\begin{align}
    \widehat{\psi}(\bx)
    &=\frac{\W(0|\bx)-\widehat{\Pr}(\hat{Y}=0|Y=1,S=0)}{\widehat{\Pr}(Y=1|S=0)}\widehat{P}_{Y|X,S=0}(1|\bx).\label{eq::empsi_fnr_proof_eq_exp2}
\end{align}
Following from \eqref{eq::psi_fnr_proof_eq_exp1}, \eqref{eq::empsi_fnr_proof_eq_exp2} and the triangle inequality, we have, under FNR, 
\begin{align}
    &\|\psi(\bx)-\widehat{\psi}(\bx)\|_p \nonumber\\
    %&=\left\|\frac{\W(0|\bx)-\Pr(\hat{Y}=0|Y=1,S=0)}{\Pr(Y=1|S=0)}P_{Y|X,S=0}(1|\bx)\right.\nonumber\\
    %&\left.\quad\quad-\frac{\W(0|\bx)-\widehat{\Pr}(\hat{Y}=0|Y=1,S=0)}{\widehat{\Pr}(Y=1|S=0)}\widehat{P}_{Y|X,S=0}(1|\bx)\right\|_p\nonumber\\
    &\leq \left\|\frac{\W(0|\bx)-\Pr(\hat{Y}=0|Y=1,S=0)}{\Pr(Y=1|S=0)}(P_{Y|X,S=0}(1|\bx)-\widehat{P}_{Y|X,S=0}(1|\bx))\right\|_p\nonumber\\
    &\quad +\left\|\widehat{P}_{Y|X,S=0}(1|\bx)\left(\frac{\W(0|\bx)-\Pr(\hat{Y}=0|Y=1,S=0)}{\Pr(Y=1|S=0)}\right.\right.\nonumber\\
    &\left.\left.\quad\quad\quad-\frac{\W(0|\bx)-\widehat{\Pr}(\hat{Y}=0|Y=1,S=0)}{\widehat{\Pr}(Y=1|S=0)}\right)\right\|_p\nonumber\\
    &\leq \left\|\frac{1}{\Pr(Y=1|S=0)}(P_{Y|X,S=0}(1|\bx)-\widehat{P}_{Y|X,S=0}(1|\bx))\right\|_p\nonumber\\
    &\quad +\left\|\frac{\W(0|\bx)-\Pr(\hat{Y}=0|Y=1,S=0)}{\Pr(Y=1|S=0)}-\frac{\W(0|\bx)-\widehat{\Pr}(\hat{Y}=0|Y=1,S=0)}{\widehat{\Pr}(Y=1|S=0)}\right\|_p\nonumber\\
    &\lesssim \left\|P_{Y|X,S=0}(1|\bx)-\widehat{P}_{Y|X,S=0}(1|\bx)\right\|_p\nonumber\\
    &\quad +\left\|\frac{\W(0|\bx)-\Pr(\hat{Y}=0|Y=1,S=0)}{\Pr(Y=1|S=0)}-\frac{\W(0|\bx)-\widehat{\Pr}(\hat{Y}=0|Y=1,S=0)}{\widehat{\Pr}(Y=1|S=0)}\right\|_p. \label{equa::Esterror_inf_one}
\end{align}
Next, we have
\begin{align}
    &\left\|\frac{\W(0|\bx)-\Pr(\hat{Y}=0|Y=1,S=0)}{\Pr(Y=1|S=0)}-\frac{\W(0|\bx)-\widehat{\Pr}(\hat{Y}=0|Y=1,S=0)}{\widehat{\Pr}(Y=1|S=0)}\right\|_p\nonumber\\
    &\leq\left\|\frac{\W(0|\bx)}{\Pr(Y=1|S=0)}-\frac{\W(0|\bx)}{\widehat{\Pr}(Y=1|S=0)}\right\|_p+\left|\frac{\Pr(\hat{Y}=0|Y=1,S=0)}{\Pr(Y=1|S=0)}-\frac{\widehat{\Pr}(\hat{Y}=0|Y=1,S=0)}{\widehat{\Pr}(Y=1|S=0)}\right|\nonumber\\
    &\leq\left|\frac{\widehat{\Pr}(Y=1|S=0)-\Pr(Y=1|S=0)}{\Pr(Y=1|S=0)\widehat{\Pr}(Y=1|S=0)}\right|\nonumber\\
    &\quad +\left|\frac{\Pr(\hat{Y}=0|Y=1,S=0)\widehat{\Pr}(Y=1|S=0)-\widehat{\Pr}(\hat{Y}=0|Y=1,S=0)\Pr(Y=1|S=0)}{\Pr(Y=1|S=0)\widehat{\Pr}(Y=1|S=0)}\right|\nonumber\\
    &\lesssim\left|\widehat{\Pr}(Y=1|S=0)-\Pr(Y=1|S=0)\right|\nonumber\\
    &\quad +\left|\Pr(\hat{Y}=0|Y=1,S=0)\widehat{\Pr}(Y=1|S=0)-\widehat{\Pr}(\hat{Y}=0|Y=1,S=0)\Pr(Y=1|S=0)\right|\nonumber\\
    &\leq 2\left|\widehat{\Pr}(Y=1|S=0)-\Pr(Y=1|S=0)\right|+\left|\widehat{\Pr}(\hat{Y}=0|Y=1,S=0)-\Pr(\hat{Y}=0|Y=1,S=0)\right|.\label{equa::proof_gen_fnr_bern_chan}
\end{align}
Combining \eqref{equa::Esterror_inf_one} and \eqref{equa::proof_gen_fnr_bern_chan} with the assumptions \eqref{eq::assume_bern_1_leq_chan} and \eqref{eq::assume_bern_2_leq_chan}, we have, for FNR,
\begin{align*}
    &\|\widehat{\psi}(\bx)-\psi(\bx)\|_p 
    \lesssim \left\|P_{Y|X,S=0}(1|\bx)-\widehat{P}_{Y|X,S=0}(1|\bx)\right\|_p.
\end{align*}
\item \textbf{DA.}
The influence function under DA is
\begin{equation*}
\begin{aligned}
     \psi(\bx)=&\left(\log\frac{\PYhata(1)\PYhatb(0)}{\PYhatb(1)\PYhata(0)}\right)\left(\W(1|\bx)-\PYhata(1)\right)\\
     &+\lambda\left(\log\frac{1-P_{S|X}(1|\bx)}{P_{S|X}(1|\bx)}-\EE{\log\frac{1-P_{S|X}(1|X)}{P_{S|X}(1|X)}\Bigg|S=0}\right).
\end{aligned}
\end{equation*}
Since $h(\bx)=P_{\hat{Y}|X}(1|\bx)$ is a given classifier, estimating
\begin{align*}
    \left(\log\frac{\PYhata(1)\PYhatb(0)}{\PYhatb(1)\PYhata(0)}\right)\left(\W(1|\bx)-\PYhata(1)\right)
\end{align*}
is more reliable than estimating
\begin{align}
    \psi_r(\bx) 
    &\defined \log\frac{1-P_{S|X}(1|\bx)}{P_{S|X}(1|\bx)}-\EE{\log\frac{1-P_{S|X}(1|X)}{P_{S|X}(1|X)}\Bigg|S=0}\nonumber\\
    &=\log\frac{1-P_{S|X}(1|\bx)}{P_{S|X}(1|\bx)}-\sum_{\bx\in\mathcal{X}}P_{X|S=0}(\bx)\log\frac{1-P_{S|X}(1|\bx)}{P_{S|X}(1|\bx)}.\label{equa::psi_r_expr}
\end{align}
Next, we bound the generalization error of estimating $\psi_r(\bx)$. Its estimator is
\begin{align}
    \widehat{\psi}_r(\bx) 
    =\log\frac{1-\widehat{P}_{S|X}(1|\bx)}{\widehat{P}_{S|X}(1|\bx)}-\sum_{\bx\in\mathcal{X}}\widehat{P}_{X|S=0}(\bx)\log\frac{1-\widehat{P}_{S|X}(1|\bx)}{\widehat{P}_{S|X}(1|\bx)}.\label{equa::est_psi_r_expr}
\end{align}
Note that, for $a,b>0$,
\begin{align}
    \left|\log\frac{a}{b}\right| \leq \frac{|a-b|}{\min\{a,b\}}.
\end{align}
Then
\begin{align}
    &\left|\log\frac{1-\widehat{P}_{S|X}(1|\bx)}{\widehat{P}_{S|X}(1|\bx)}-\log\frac{1-P_{S|X}(1|\bx)}{P_{S|X}(1|\bx)}\right|\nonumber\\
    &\leq |\widehat{P}_{S|X}(1|\bx)-P_{S|X}(1|\bx)|\left(\frac{1}{\min\{\widehat{P}_{S|X}(1|\bx),P_{S|X}(1|\bx)\}} + \frac{1}{\min\{1-\widehat{P}_{S|X}(1|\bx),1-P_{S|X}(1|\bx)\}}\right)\nonumber\\
    &\leq |\widehat{P}_{S|X}(1|\bx)-P_{S|X}(1|\bx)| \frac{2}{m_X},\label{equa::log_Phat_Pleql1}
\end{align}
where $m_X$ is a constant number:
\begin{align*}
    &m_X \defined \\
    &\min\left\{\left\{\widehat{P}_{S|X}(1|\bx)|\bx\in\mathcal{X}\right\}\cup \left\{P_{S|X}(1|\bx)|\bx\in\mathcal{X}\right\}\cup\left\{1-\widehat{P}_{S|X}(1|\bx)|\bx\in\mathcal{X}\right\}\cup \left\{1-P_{S|X}(1|\bx)|\bx\in\mathcal{X}\right\}\right\}.
\end{align*}
Also of note, for any $\bx\in\mathcal{X}$,
\begin{align}
    \left|\log\frac{1-\widehat{P}_{S|X}(1|\bx)}{\widehat{P}_{S|X}(1|\bx)}\right|
    \leq \frac{\left|1-2\widehat{P}_{S|X}(1|\bx)\right|}{\min\left\{\widehat{P}_{S|X}(1|\bx),1-\widehat{P}_{S|X}(1|\bx)\right\}}\leq \frac{1}{m_X}.\label{equa::log_one_min_Phat_Phat}
\end{align}
Combining \eqref{equa::psi_r_expr} and \eqref{equa::est_psi_r_expr} with \eqref{equa::log_Phat_Pleql1} and \eqref{equa::log_one_min_Phat_Phat}, we have
\begin{align*}
    &\left|\widehat{\psi}_r(\bx)-\psi_r(\bx)\right|\\
    &\leq \frac{2}{m_X}\left|\widehat{P}_{S|X}(1|\bx)-P_{S|X}(1|\bx)\right| + \frac{1}{m_X} \sum_{\bx \in \mathcal{X}} \left|\widehat{P}_{X|S=0}(\bx)-P_{X|S=0}(\bx)\right| \\
    &\quad + \frac{2}{m_X} \sum_{\bx \in \mathcal{X}} \left|\widehat{P}_{S|X}(1|\bx)-P_{S|X}(1|\bx)\right|P_{X|S=0}(\bx)\\
    &=\frac{2}{m_X}\left|\widehat{P}_{S|X}(1|\bx)-P_{S|X}(1|\bx)\right| + \frac{1}{m_X}  \left\|\widehat{P}_{X|S=0}-P_{X|S=0}\right\|_1 + \frac{2}{m_X} \left\|\widehat{P}_{S|X}(1|\bx)-P_{S|X}(1|\bx)\right\|_1.
\end{align*}
Therefore,
\begin{align*}
    &\left\|\widehat{\psi}_r(\bx)-\psi_r(\bx)\right\|_p \\
    &\leq \frac{2}{m_X}\left\|\widehat{P}_{S|X}(1|\bx)-P_{S|X}(1|\bx)\right\|_p + \frac{1}{m_X}  \left\|\widehat{P}_{X|S=0}-P_{X|S=0}\right\|_1 + \frac{2}{m_X} \left\|\widehat{P}_{S|X}(1|\bx)-P_{S|X}(1|\bx)\right\|_1.
\end{align*}
Based on the assumption: $\left\|\widehat{P}_{X|S=0}-P_{X|S=0}\right\|_1 \lesssim \left\|\widehat{P}_{S|X}(1|\bx)-P_{S|X}(1|\bx)\right\|_1$, we have
\begin{align*}
    \left\|\widehat{\psi}_r(\bx)-\psi_r(\bx)\right\|_p
    &\lesssim \left\|\widehat{P}_{S|X}(1|\bx)-P_{S|X}(1|\bx)\right\|_p + \left\|\widehat{P}_{S|X}(1|\bx)-P_{S|X}(1|\bx)\right\|_1\\
    &\lesssim \left\|\widehat{P}_{S|X}(1|\bx)-P_{S|X}(1|\bx)\right\|_p.
\end{align*}
Hence, for DA,
\begin{align*}
    \|\widehat{\psi}(\bx)-\psi(\bx)\|_p \lesssim \left\|\widehat{P}_{S|X}(1|\bx)-P_{S|X}(1|\bx)\right\|_p.
\end{align*}
\end{enumerate}

Proposition~\ref{prop::generalizationthm_Cal} follows from Lemma~\ref{lem::gen_bound} and the following large deviation results by \citet{weissman2003inequalities}. For all $\epsilon > 0$,
\begin{equation*}
    \Pr\left(\|\widehat{P}-P\|_1 \geq \epsilon\right) 
    \leq (2^M-2)\exp\left(-n\bar{\phi}(\pi_P)\epsilon^2/4\right),
\end{equation*}
where $P$ is a probability distribution on the set $[M]$, $\widehat{P}$ is the empirical distribution obtained from $n$ i.i.d. samples,
$\pi_P \defined \max_{\mathcal{M}\subseteq[M]}\min(P(\mathcal{M}),1-P(\mathcal{M}))$,
\begin{equation*}
    \bar{\phi}(p)\defined 
    \begin{cases}
    \frac{1}{1-2p}\log \frac{1-p}{p} & p\in[0,1/2), \\ 
    2 & p=1/2,
    \end{cases}
\end{equation*}
and $\|\widehat{P}-P\|_1 \defined \sum_{x\in\mathcal{X}}|\widehat{P}(x)-P(x)|$.
Note that $\bar{\phi}(\pi_P)\geq 2$ which implies that
\begin{equation}
\label{equa:L1bound_weissman}
\begin{aligned}
    \Pr\left(\|\widehat{P}-P\|_1 \geq \epsilon\right)
    &\leq \exp(M)\exp(-n\epsilon^2/2).
\end{aligned}
\end{equation}
Hence, by taking $P=P_{Y,X|S=0}$, $M=|\mathcal{Y}||\mathcal{X}|=2|\mathcal{X}|$ and $\epsilon = \sqrt{\frac{2}{n}\left(M-\log\beta\right)}$, Inequality~\eqref{equa:L1bound_weissman} implies that, with probability at least $1-\beta$,
\begin{equation}
\label{equa:L1bound_thm}
    \left\|\widehat{P}_{Y,X|S=0}-P_{Y,X|S=0}\right\|_1 \leq \sqrt{\frac{2}{n}\left(2|\mathcal{X}|-\log\beta\right)},
\end{equation}
where $\widehat{P}_{Y,X|S=0}$ is the empirical distribution obtained from $n$ i.i.d. samples. Similarly, with probability at least $1-\beta$,
\begin{equation}
    \left\|\widehat{P}_{S,X}-P_{S,X}\right\|_1 \leq \sqrt{\frac{2}{n}\left(2|\mathcal{X}|-\log\beta\right)}.
\end{equation}

Let $\widehat{P}_{Y|X,S=0}=\frac{\widehat{P}_{Y,X|S=0}}{\widehat{P}_{X|S=0}}$ be the empirical conditional distribution obtained from $n$ i.i.d. samples. Then, for the disparity metrics in Table~\ref{tabel:ExpDiscMet} except DA, with probability at least $1-\beta$,
\begin{align*}
    \left\|\widehat{\psi}(\bx)-\psi(\bx)\right\|_1
    &\lesssim \left\|\widehat{P}_{Y|X,S=0}(1|\bx)-P_{Y|X,S=0}(1|\bx)\right\|_1\\
    &\lesssim \left\|\widehat{P}_{Y,X|S=0}-P_{Y,X|S=0}\right\|_1\\
    &\lesssim \sqrt{\frac{1}{n}\left(|\mathcal{X}|-\log\beta\right)}.
\end{align*}
Here the second inequality holds true because
\begin{align*}
    &\left\|\widehat{P}_{Y|X,S=0}(1|\bx)-P_{Y|X,S=0}(1|\bx)\right\|_1\\
    &=\sum_{\bx\in\mathcal{X}} P_{X|S=0}(\bx)\left|\widehat{P}_{Y|X,S=0}(1|\bx)-P_{Y|X,S=0}(1|\bx)\right|\\
    &\leq \left\|\widehat{P}_{Y,X|S=0}-P_{Y,X|S=0}\right\|_1+\sum_{\bx\in\mathcal{X}} \widehat{P}_{Y|X,S=0}(1|\bx)\left|\widehat{P}_{X|S=0}(\bx)-P_{X|S=0}(\bx)\right|\\
    &\leq \left\|\widehat{P}_{Y,X|S=0}-P_{Y,X|S=0}\right\|_1+\left\|\widehat{P}_{X|S=0}-P_{X|S=0}\right\|_1
    \lesssim \left\|\widehat{P}_{Y,X|S=0}-P_{Y,X|S=0}\right\|_1.
\end{align*}
Similar analysis also holds for DA.

\end{proof}

%\clearpage
\section{Supporting Experimental Results}
\label{Appendix::Experiments}

\subsection{Experiments on Synthetic Datasets}

\subsubsection*{Descent Procedure}

\noindent \textit{Setup}: We consider a toy problem with 3 binary variables $X = (X_1,X_2,X_3).$ We define $p_i = \Pr(X_i=1|S=0)$ and $q_i = \Pr(X_i=1|S = 1)$, and assume that:
\begin{align*}
(p_1,p_2,p_3)&=(0.9,0.2,0.2)\\
(q_1,q_2,q_3)&=(0.1,0.5,0.5)
\end{align*}
Given any value of $X$, we draw the value of $Y$ for using the same distribution for each group, namely: $$P_{Y|X,S=0}(1|\bx) = P_{Y|X,S=1}(1|\bx) = \mathsf{logistic}(5x_1-2x_2-2x_3).$$
We train a logistic regression model over 50k samples. We randomly draw 12.5k samples for the auditing dataset and 12.5k samples for the holdout dataset, and apply the descent procedure in Algorithm~\ref{alg:descent} for the FPR metric. At each step, the influence function is computed on the auditing dataset, and applied to both the auditing and the holdout set. 

\noindent \textit{Results}: As shown in Figure~\ref{Fig::toy_descent_fpr}, the procedure converges to a counterfactual distribution after around 40 iterations (we show additional steps for the sake of illustration). In practice, a stopping rule can be designed to stop the descent procedure based on number of iterations or a target discrimination gap value. Then we use the proposed preprocessor to map samples from $S=0$ to new samples. Then the value of FPR decreases from 29.1\% to 4.1\%.

\subsubsection*{Joint Proxies}
\label{sec::toy_example_jointproxy}

\noindent \textit{Setup}: We consider a simple experiment to show that the preprocessor mitigates discrimination while removing a single proxy variable does not. We consider a setting where $X=(X_1,X_2,X_3)\in\{-1,1\}^3$ and choose the joint distribution matrices of $(X_1,X_2)$ for $S=0$ and $S=1$ as
\begingroup
\small
\begin{equation}
    \mathbf{P}_0 = \left(\begin{matrix} 0.60 & 0.00 \\ 0.25 & 0.15 \end{matrix}\right), 
    \mathbf{P}_1 = \left(\begin{matrix} 0.05 & 0.00 \\ 0.20 & 0.75 \end{matrix}\right).
\end{equation}
\endgroup
Then we choose $X_3$ to be independent of $(X_1,X_2)$ with $\Pr(X_3=1|S=i)=0.3$ for $i=0,1$. We draw the values of $Y$ according to $P_{Y|X,S=i}(1|\bx)=\mathsf{logistic}(6x_1x_2+x_3)$ for $i=0,1$, and fit a logistic regression using 50k samples.

\noindent \textit{Results}: The value of $\textrm{DA}_{0}$ is 14.0\%. In this case, both $X_1$ and $X_2$ are proxy variable. We remove $X_1$ from dataset and retrain a logistic regression as a classifier. It turns out that the value of $\textrm{DA}_{0}$ becomes larger: 24.8\%. This is because the pair $(X_1,X_2)$ is a joint proxy and, consequently, removing one of them could not reduce discrimination.

Next, we apply Algorithm~\ref{alg:descent} and the proposed preprocessor to decrease discrimination. For the sake of example, we randomly draw 12.5k new samples for the auditing dataset and 12.5k samples for the holdout dataset, and apply the descent procedure in Algorithm~\ref{alg:descent} under $\textrm{DA}_{0}$. At each step, the influence function is computed on the auditing dataset, and applied to both the auditing and the holdout set. Then we use the preprocessor to map samples from $S=0$ to new samples and $\textrm{DA}_{0}$ becomes 0.0\%.

\end{document}